\documentclass[runningheads]{llncs}

\newif\ifdraft
\draftfalse  %

\usepackage[T1]{fontenc}
\usepackage{graphicx}
\usepackage{booktabs}
\usepackage[misc]{ifsym}
\newcommand{\corr}{(\Letter)}
\usepackage{algorithm}
\usepackage{algpseudocode}
\usepackage{csquotes}
\usepackage{paralist}
\usepackage{hyperref}
\usepackage{amsmath}
\usepackage{amssymb}
\usepackage{mathtools}
\usepackage[capitalize,noabbrev]{cleveref}
\usepackage{bm}         %
\usepackage{siunitx}    %
\usepackage{dsfont}     %
\usepackage{xspace}     %
\usepackage{xifthen}    %
\usepackage{caption}
\usepackage{subcaption}
\usepackage{enumitem}
\usepackage{todonotes}
\usepackage{float}
\usepackage{longtable}  %
\usepackage{makecell}   %

\usepackage{bibunits}
\defaultbibliographystyle{splncs04}
\defaultbibliography{bibliographies/shortRS,bibliographies/automl_bibliograhpy/shortstrings,bibliographies/automl_bibliograhpy/lib,bibliographies/automl_bibliograhpy/shortproc}
\makeatletter

\def\hyper@natlinkstart#1{%
  \Hy@backout{#1}%
  \hyper@linkstart{cite}{cite.\@bibunitname.#1}%
  \def\hyper@nat@current{#1}%
}

\def\hyper@natlinkbreak#1#2{%
  \hyper@linkend#1\hyper@linkstart{cite}{cite.\@bibunitname.#2}%
}

\def\hyper@natanchorstart#1{%
  \hyper@anchorstart{cite.\@bibunitname.#1}%
}

\def\bibcite#1#2{%
  \@newl@bel{b}{#1}{\hyper@@link[cite]{}{cite.\@bibunitname.#1}{#2}}%
}

\def\@lbibitem[#1]#2{%
  \@skiphyperreftrue
  \H@item[\hyper@anchorstart{cite.\@bibunitname.#2}%
    \@BIBLABEL{#1}\hyper@anchorend\hfill]%
  \@skiphyperreffalse
  \if@filesw
    \begingroup
      \let\protect\noexpand
      \immediate\write\@auxout{\string\bibcite{#2}{#1}}%
    \endgroup
  \fi
  \ignorespaces
}
\def\@bibitem#1{%
  \@skiphyperreftrue\H@item\@skiphyperreffalse
  \hyper@anchorstart{cite.\@bibunitname.#1}\relax\hyper@anchorend
  \if@filesw
    \begingroup
      \let\protect\noexpand
      \immediate\write\@auxout{%
        \string\bibcite{#1}{\the\value{\@listctr}}%
      }%
    \endgroup
  \fi
  \ignorespaces
}
\def\@citex[#1]#2{%
  \let\@citea\@empty
  \@cite{%
    \@for\@citeb:=#2\do{%
      \@citea
      \def\@citea{,\penalty\@m\ }%
      \edef\@citeb{\expandafter\@firstofone\@citeb}%
      \if@filesw
        \immediate\write\@auxout{\string\citation{\@citeb}}%
      \fi
      \@ifundefined{b@\@citeb}{%
        \mbox{\reset@font\bfseries ?}%
        \G@refundefinedtrue
        \@latex@warning{Citation `\@citeb' on page \thepage \space undefined}%
      }{%
        \hyper@natlinkstart{\@citeb}%
          \hbox{\csname b@\@citeb\endcsname}%
        \hyper@natlinkend
      }%
    }%
  }{#1}%
}
\makeatother
\newcolumntype{R}[1]{>{\raggedright\arraybackslash}p{#1}}

\ifdefined\N
\renewcommand{\N}{\mathds{N}} %
\else \newcommand{\N}{\mathds{N}} \fi
\newcommand{\R}{\mathds{R}} %
\ifdefined\C
\renewcommand{\C}{\mathds{C}} %
\else \newcommand{\C}{\mathds{C}} \fi
\newcommand{\argmin}{\mathop{\mathrm{arg\,min}}} %
\newcommand{\argmax}{\mathop{\mathrm{arg\,max}}} %
\newcommand{\diag}{\operatorname{diag}} %
\renewcommand{\P}{\mathds{P}} %
\newcommand{\E}{\mathds{E}} %
\newcommand{\normal}{\mathcal{N}} %
\newcommand{\Xspace}{\mathcal{X}} %
\newcommand{\Yspace}{\mathcal{Y}} %
\newcommand{\xv}{\mathbf{x}} %
\newcommand{\yv}{\mathbf{y}} %
\newcommand{\allDatasets}{\mathds{D}} %
\newcommand{\D}{\mathcal{D}} %
\newcommand{\xyi}[1][i]{\left(\xv^{(#1)}, y^{(#1)}\right)} %
\newcommand{\defAllDatasets}{\bigcup_{n \in \N}(\Xspace \times \Yspace)^n} %
\renewcommand{\xi}[1][i]{\xv^{(#1)}} %
\newcommand{\Lam}{\bm{\Lambda}}	 %
\newcommand{\preimageInducerShort}{\allDatasets\times\Lam} %
\newcommand{\ind}{\mathcal{I}} %

\newcommand{\fx}{f(\xv)} %
\newcommand{\fdomains}{f: \Xspace \rightarrow \R^g} %
\newcommand{\Hspace}{\mathcal{H}} %
\newcommand{\yh}{\hat{y}} %
\newcommand{\Lxy}{L\left(y, \fx\right)} %
\newcommand{\risk}{\mathcal{R}} %

\newcommand{\riske}{\mathcal{R}_{\text{emp}}} %
\newcommand{\pert}[3]{\ifthenelse{\equal{#2}{}}{\tilde{#1}}{\ifthenelse{\equal{#3}{}}{\tilde{#1}^{#2}}{\tilde{#1}^{#2|#3}}}}	%

\newcommand{\kv}{\mathbf{k}} %
\newcommand{\kxxp}{k\left(\xv, \xv^{\prime} \right)} %
\newcommand{\kxij}[2]{k\left(\xi, \xi[j] \right)} %
\newcommand{\Kmat}{\mathbf{K}} %
\newcommand{\RashomonSet}{Rashomon set}
\newcommand{\CashomonSet}{CASHomon set}
\newcommand{\RashomonSets}{Rashomon sets}
\newcommand{\CashomonSets}{CASHomon sets}
\newcommand{\Cashomon}{CASHomon}
\newcommand{\Rset}{\mathcal{RS}} %
\newcommand{\fref}{f_{\text{ref}}} %
\newcommand{\Rsetfull}{\Rset(\varepsilon,\fref,\mathcal{F})} %
\newcommand{\Cset}{\mathcal{CS}} %
\newcommand{\Csetfull}{\Cset(\varepsilon,\fref,\Hspace_\text{CASH}^{\D})} %
\newcommand{\PXY}{P_{\Xspace\Yspace}} %
\newcommand{\TruVar}{\textsc{TruVaR}\xspace}
\newcommand{\TruVarImp}{\textsc{TruVaRImp}\xspace}
\newcommand{\TreeFARMS}{\textsc{TreeFARMS}\xspace}
\newcommand{\evalpointbar}{\mbox{$\lambda'$}}
\newcommand{\deltabar}{\overline{\delta}}
\newcommand{\candset}{\tilde{\Lam}}
\newcommand{\evalpoint}{\lambda}
\newcommand{\objfun}{c}
\newcommand{\objobs}{\hat{c}}

\newcommand{\Prob}{\mathrm{Prob}}

\usepackage{centernot}

\makeatletter
\newcommand{\printfnsymbol}[1]{%
  \textsuperscript{\@fnsymbol{#1}}%
}
\makeatother

\begin{document}
\begin{bibunit}[splncs04]

\title{CASHomon Sets: Efficient Rashomon Sets Across Multiple Model Classes and their Hyperparameters}

\titlerunning{CASHomon Sets}
\author{Fiona Katharina Ewald\thanks{Equal contribution}\inst{1,2}\orcidID{0009-0002-6372-3401} \and
Martin~Binder\printfnsymbol{1}\inst{1,2}\and
Bernd~Bischl\inst{1,2}\orcidID{0000-0001-6002-6980} \and
Matthias~Feurer\inst{3,4}\orcidID{0000-0001-9611-8588} \and
Giuseppe~Casalicchio\inst{1,2}\orcidID{0000-0001-5324-5966}\corr}
\authorrunning{F. K. Ewald, M. Binder et al.}

\institute{Department of Statistics, LMU Munich, Munich, Germany
\Letter\,\email{Giuseppe.Casalicchio@stat.uni-muenchen.de} \and
Munich Center for Machine Learning (MCML), Munich, Germany \and
TU Dortmund University, Dortmund, Germany \and
Lamarr Institute for Machine Learning and Artificial Intelligence, Dortmund, Germany}

\maketitle              %

\begin{abstract}
    \RashomonSets{} are model sets within one model class that perform nearly as well as a reference model from the same model class.
    They reveal the existence of alternative well-performing models, which may support different interpretations.
    This enables selecting models that match domain knowledge, hidden constraints, or user preferences.
    However, efficient construction methods currently exist for only a few model classes.
    Applied machine learning usually searches many model classes, and the best class is unknown beforehand.
    We therefore study Rashomon sets in the combined algorithm selection and hyperparameter optimization (CASH) setting and call them \emph{\CashomonSets{}}.
    We propose \TruVarImp, a model-based active learning algorithm for level set estimation with an implicit threshold, and provide convergence guarantees.
    On synthetic and real-world datasets, \TruVarImp reliably identifies \CashomonSet{} members and matches or outperforms naive sampling, Bayesian optimization, classical and implicit level set estimation methods, and other baselines.
    Our analyses of predictive multiplicity and feature-importance variability across model classes question the common practice of interpreting data through a single model class.

\keywords{Rashomon Effect \and Level-set Estimation \and Feature Importance \and Predictive Multiplicity \and Automated Machine Learning}
\end{abstract}
\section{Introduction} \label{sec_intro}

The \emph{Rashomon effect} \cite{breiman_statistical_2001} refers to the existence of multiple models with comparable predictive performance that capture different aspects of the data. 
Serving as a foundational tool for systematically studying this effect, a \emph{\RashomonSet{}} \cite{fisher-jmlr19a} is a set of models that perform almost as well as a reference model.
Recently, Rudin et al.~\cite{rudin-icml24a} emphasized its practical yet underestimated importance in applied \emph{machine learning}~(ML), noting its implications and impact on
\begin{inparaenum}[\bf (1)]
    \item users' flexibility in selecting models based on domain knowledge or preferences, such as fairness constraints; \label{RE_field1}
    \item the uncertainty of model-based summary statistics; and \label{RE_field2}
    \item the reliability of explanations in \emph{interpretable machine learning}~(IML). \label{RE_field3}
\end{inparaenum}
Unfortunately, efficient \RashomonSet{} construction is currently possible for only a few model classes.
Yet different model classes are suited for different problems, and the best class is rarely known in advance \cite{fernandez-jmlr14a,shwartz-ziv-if22a,mcelfresh-neurips23a}. %
\emph{Hyperparameter optimization}~(HPO, \cite{bischl-dmkd23a}) and \emph{automated ML}~(AutoML, \cite{thornton-kdd13a,feurer-nips15a,hutter-book19a}) address this via data-driven model selection.
But this often increases complexity: AutoML systems, in their main and somewhat exclusive pursuit of maximizing predictive performance, select flexible, non-parametric models or ensembles, often ignoring interpretability and other constraints. As a result, the models produced by an AutoML system are typically complex and difficult to understand.
To address this gap, we contribute the following:
\begin{enumerate}[leftmargin=*,label=\textbf{(\arabic*)}, nosep]
    \item We formally introduce \CashomonSets{} (Section~\ref{sec_formal_framework}), which extend \RashomonSets{} beyond a single model class in the \emph{combined algorithm selection and hyperparameter optimization}~(CASH) setting.
    \item We introduce \TruVarImp, a model-based active learning algorithm for \emph{level set estimation}~(LSE) with an implicitly defined threshold (Section~\ref{sec_algo}), which can be viewed as a Bayesian optimization algorithm, and provide theoretical accuracy guarantees.
    \item We empirically show that \TruVarImp{} efficiently finds \CashomonSets{} and performs competitively with, or better than, naive sampling, AutoML-based pipelines, and other LSE baselines on nine datasets (Section~\ref{sec_experiments}).
    \item We further study \CashomonSets{} from an application perspective (Section~\ref{sec:fi_for_RE}), finding that predictive multiplicity, quantified by Rashomon capacity \cite{hsu-neurips22a}, which we extend to regression tasks, can differ between Rashomon and \CashomonSets{}, and that \emph{feature importance}~(FI) values can vary substantially across model classes, questioning the common practice of basing interpretations on a single model class.
\end{enumerate} %

\section{Related Work} \label{sec_rw}

We position our work at the intersection of (i) \RashomonSets{} for predictive multiplicity and interpretability, and (ii) LSE for efficient identification of near-optimal regions in \emph{hyperparameter}~(HP) spaces.

\paragraph{Rashomon effect in interpretability and predictive multiplicity.}
Prior work shows that \RashomonSets{} often contain models with desirable properties (e.g., interpretability or fairness) without substantial performance loss \cite{semenova-facct22a,rudin-icml24a}.
Rashomon-based analyses have also been used to study predictive multiplicity and the instability of explanations (e.g., feature importance) derived from predictive models \cite{hsu_arxiv24,hsu-nips2024,mueller-ecml23a,Kobylinska-ieee24a}. 
We analyze both aspects: we study predictive multiplicity using the \emph{Rashomon capacity} \cite{hsu-neurips22a} and the variability of feature importance across well-performing models using \emph{variable importance clouds}~(VICs, \cite{dong-nmi20a}).
 
However, many existing works on \RashomonSets{} either study them theoretically or develop construction or exploration methods for specific model classes. 
Examples include ridge regression and related linear-model settings \cite{dong-nmi20a,semenova-facct22a}, sparse GAMs \cite{zhong-neurips23a}, rule lists \cite{mata_arXiv22}, rule sets \cite{ciaperoni-acm24}, neural networks through dropout-based exploration \cite{hsu_arxiv24}, gradient boosting \cite{hsu-nips2024}, and for additive models, kernel ridge regression, and random forests \cite{laberge-jmlr23a}. 
For generalized and scalable optimal sparse decision trees (GOSDT) \cite{lin-icml20a}, a prominent approach to construct a \RashomonSet{} is \TreeFARMS \cite{xin-neurips22a}, which we use as a baseline in our empirical comparison (Section~\ref{sec_experiments}). 
Another line of work uses candidate models generated by AutoML systems (e.g., H2O AutoML~\cite{ledell-automl20a}) as a proxy \RashomonSet{} that includes multiple model classes and studies the variability of explanations \cite{cavus-isf26a}. 
However, these approaches either remain tied to a specific model class or rely on first generating a large pool of candidate models and then subsetting it retrospectively, rather than directly identifying such sets during search. 
In contrast, we formalize \RashomonSets{} directly in the CASH setting, where the model class itself is part of the search space, and propose an efficient algorithm for identifying candidate models across heterogeneous model classes.

\paragraph{Level set estimation and hyperparameter optimization.}
LSE\footnote{While Gotovos et al.~\cite{gotovos2013} also call their algorithm \enquote{LSE}, we use the term for the general problem and denote their algorithm in monospace font, \texttt{LSE}.} 
\cite{gotovos2013} identifies the set of inputs for which an unknown function is above or below a given threshold.
We frame \CashomonSet{} construction as LSE, where candidates are classified by whether performance lies within a tolerance of the optimum.
Since characterizing such sets over a continuous domain is usually infeasible, LSE methods classify a finite candidate subset. 
This differs from standard hyperparameter optimization (HPO), which targets one near-optimal configuration.

Bryan et al.~\cite{bryan2005active} studied active LSE for expensive black-box functions using Gaussian processes (GPs)~\cite{rasmussen-book06a}.
Their \emph{straddle heuristic} samples where the threshold lies within the confidence region of the function value and prioritizes points where the distance between the confidence region bounds and the threshold value is largest.
Gotovos et al.~\cite{gotovos2013} refined this by intersecting confidence regions across iterations and discarding points once the intersection no longer crosses the threshold.
Crucially, they also propose \emph{implicitly defined} thresholds relative to the unknown optimum.
This maps directly to finding \CashomonSets{} (and \RashomonSets{}) when the best model serves as a reference but is unknown.
Bogunovic et al.~\cite{bogunovic-nips16a} propose \TruVar, an LSE algorithm that handles pointwise evaluation costs and heteroscedastic noise.
Instead of sampling only by confidence width, \TruVar selects points that most reduce confidence widths for all points not yet classified to below a shrinking target value.
Among GP-based LSE methods, \TruVar{} is the most direct methodological precursor to our approach.
We extend \TruVar{} to implicit thresholds using the method of Gotovos et al.~\cite{gotovos2013} in \TruVarImp{} (Section~\ref{sec_algo}).
Concretely, \TruVarImp{} jointly tracks potential minimizers and threshold-uncertain points, then allocates evaluations to reduce uncertainty in both sets.

\section{Formalization of CASHomon Sets} \label{sec_formal_framework}
Let $\ind:\preimageInducerShort\rightarrow\Hspace$ be an \emph{inducer} which maps data $\D\sim (\PXY)$ and a \emph{hyperparameter configuration}~(HPC) $\lambda\in\Lam$ to a fitted \emph{model}
$\fdomains$, where $g$ equals~$1$ for regression and the number of classes for classification.
The $i$-th observation in $\D$ is $\xyi, i \in \{1,...,n\}$, and the set of all datasets is $\allDatasets=\defAllDatasets$, where $\Xspace$ is the feature space and $\Yspace$ is the target space.
Let $L:\Yspace\times\R^g\rightarrow\R,$ be a pointwise loss function and $\risk(f) = \E_{(\xv,y)~\sim~\PXY}[\Lxy]$ be the expected risk, approximated via the empirical risk $\riske(f)=\frac{1}{n} \sum_{i=1}^n L(y^{(i)},f(\xv^{(i)})).$

\subsection{Background on CASH}\label{sec_CASH}

In the above formalism, HPs can represent all decisions made when producing a model (or potentially a full ML pipeline).
This can include, in particular, choosing and parameterizing data preprocessing, selecting features, and choosing a learning algorithm.
Letting categorical HPs represent discrete algorithmic decisions (i.e., which method to select) is a common way to reduce the CASH problem to an HPO problem \cite{thornton-kdd13a,feurer-nips15a}. 
This work concentrates on joint optimization of the model class and its associated HPs.
We use the term \emph{model class} to refer to the set of models $\Hspace_m := \operatorname{range}(\ind_m)$ produced by the same inducer $\ind_m$ with HP space $\Lam_m$.
Thus, this captures the intuition that a model class is a set of models that share a common structure, e.g., linear models (LMs), decision trees (DTs), or neural networks (NNs).
We generally have multiple different model classes, indexed by $m \in \{1,\ldots,M\}$.
We can now consider the overall modeling process as a single inducer, which is configured simultaneously by $m$ and $\lambda_m$, where the latter is subordinate to the former (i.e., a typical hierarchical HPC structure, commonly encountered in CASH):
$\ind_{\text{CASH}}(\D,(m,\lambda_m)) = \ind_m(\D,\lambda_m)$
with $\Lam_{\text{CASH}}:=\bigcup^M_{m=1} \{(m, \lambda_m) \mid \lambda_m \in \Lam_m\}$ as the corresponding HP space.
We denote the union of all model classes as the \emph{CASH hypothesis space} $\Hspace_{\text{CASH}}:=\bigcup^M_{m=1} \Hspace_m$, i.e., the set of all models that can be produced by the overall modeling process when any one of the $M$ model classes can be used.

\subsection{From Rashomon Sets to CASHomon Sets} \label{sec_cashomonsets}

A \RashomonSet{} is a set of measurable functions from $\Xspace$ to $\mathbb{R}^g$ containing models that perform nearly equally well as %
a reference model $\fref$ \cite{fisher-jmlr19a}. Formally, 
$ \Rsetfull:=\left\{f\in\mathcal{F}\mid \risk(f)\leq \risk(\fref) + \varepsilon\right\},$
where the offset $\varepsilon>0$ can either be given in absolute terms, or as a value relative to $\risk(\fref)$; in general, we write $\varepsilon=\varepsilon_{\mathrm{abs}}+\varepsilon_{\mathrm{rel}}\,\risk(\fref)$, assuming $\risk(\fref)>0$.
In practice, $\fref$ is typically estimated as the empirical risk minimizer from the considered hypothesis space, yielding the \emph{empirical \RashomonSet{}}.
Prior work restricts $\mathcal{F}$ to a specific model class to make the \RashomonSet{} estimation feasible (see Section~\ref{sec_rw}).
Instead, we do not restrict $\mathcal{F}$ to a single model class and approximate $\mathcal{F}$ by the set of models produced when fitting a model to the given dataset and varying its HPs. 
For a single model class with inducer $\ind_m$ and HP space $\Lam_m$, we specify the set of models $\mathcal{F}$ obtainable by tuning $\lambda_m$ for a given dataset $\D \in \allDatasets$ as
\begin{align} \label{eq_F_HP}
    \mathcal{F} = \Hspace^{\D}_{m,\text{ HPO}}:= \operatorname{range}(\ind_m(\D,\cdot)) = \{\ind_m(\D,\lambda_m) \mid \, \lambda_m \in \Lam_m\}.
\end{align}
As a necessary consequence of restricting the function class to $\Hspace^{\D}_{m,\text{ HPO}}$, the \emph{HPO \RashomonSet{}} $\Rset(\varepsilon,\fref,\Hspace^{\D}_{m,\text{ HPO}})$ is generally a strict subset of $\Rset(\varepsilon,\fref,\Hspace_m)$.
We further discuss this restriction and the impact of having only finite data in Appendix \ref{appendix_RS_HPO_viz}. 
The size of this HPO \RashomonSet{} generally increases with the number of HPs that $\ind_m$ has, and, therefore, naturally lends itself to be combined with the CASH approach (Section~\ref{sec_CASH}).
So while using $\Hspace^{\D}_{m,\text{ HPO}}$ instead of $\Hspace_m$ is a restriction, our approach also allows for a larger and more diverse \RashomonSet{} by considering multiple model classes at the same time when using the CASH search space. %
We define the \emph{\CashomonSet{}} as
\begin{align*}
    \Csetfull :=\left\{f\in\Hspace_{\text{CASH}}^{\D}\mid \risk(f)\leq \risk(\fref) +\varepsilon\right\},
\end{align*}
where $    \Hspace_{\text{CASH}}^{\D}:
    =\bigcup^M_{m=1} \{\ind_m(\D,\lambda_m) \, \vert \, \lambda_m \in \Lam_m\}. 
$
When the dataset $\D$ is fixed, we say that an HPC $\lambda_m \in \Lam_m$ is in the \RashomonSet{} if, for $m \in \{1,\ldots,M\}$, $\ind_m(\D,\lambda_m)$ is in the \RashomonSet{}.
Although $M$ model classes are considered in $\Hspace_{\text{CASH}}^{\D}$, not all of them will be represented in the \CashomonSet{} if some model classes outperform others, i.e., model class selection is performed on the fly.

Due to the computational complexity of finding a compact description of the exact \CashomonSet{}, we approximate it using a random subsample of models inside that set, which is sufficient for analysis methods that operate on discrete \RashomonSets{} or require integration over them.
This aligns with the literature, where \RashomonSets{} are often approximated by sampling from them (e.g.,~\cite{hsu-nips2024,hsu_arxiv24}).
Given $\D$ and a proposal distribution $P_{\text{CASH}}$ on $\Lam_{\text{CASH}}$, our goal is to sample from the \CashomonSet{} through rejection sampling:
\begin{inparaenum}[\bf (1)]
    \item Sample a candidate set of size $C$, $\candset=\{\lambda_i\}_{i=1}^C \sim (P_{\text{CASH}})^{C}$;
    \item Evaluate or predict the (empirical) risk $\riske(\ind(\D,\lambda_i))$ for all $i$;
    \item Accept all $\lambda_i$ for which $\risk_\mathrm{emp}(\ind(\D,\lambda_i)) \leq \risk_\mathrm{emp}(\fref)+\varepsilon$ and reject all others.
\end{inparaenum}
Aiming for a representative sample from $P_{\text{CASH}}\left(\lambda\mid \ind(\D,\lambda) \in \Cset\right)$ is a natural goal for characterizing $\Cset$ as over- or undersampling particular regions could distort downstream analyses.

\section{An Efficient Algorithm for Finding CASHomon Sets} \label{sec_algo}

Algorithmically, deciding which models are in the \CashomonSet{} reduces, in principle, to an expensive, black-box level-set estimation problem, possibly on a mixed and hierarchical space.
Direct application of LSE procedures (using appropriate surrogate models due to their computational cost and the complex search-space structure) is complicated by the fact that the level-set threshold is in our case only implicitly defined \cite{gotovos2013} as $\risk(\fref) + \varepsilon$, where $\fref$ is the best performing model from $\Hspace_{\text{CASH}}^{\D}$, which also needs to be estimated.
While one could first perform HPO to identify a reference model, followed by a (surrogate-based) LSE algorithm to identify models in the \CashomonSet{}, we introduce an algorithm that combines these steps, yielding greater efficiency.

We propose the \TruVarImp algorithm, which extends the \TruVar algorithm~\cite{bogunovic-nips16a} to handle implicitly defined threshold levels.
For (direct, explicit) LSE, \TruVar maintains a finite set of candidate points to classify as lying above or below the given threshold.
It uses a GP model that predicts both the posterior mean $\mu(\lambda)$ and the posterior variance $\sigma^2(\lambda)$ of the objective value at each candidate point $\lambda$.
In each iteration $t$, candidates with posterior distributions that fall to one side of the threshold with high confidence are classified and removed from the active, unclassified set $U_t$ and instead added to the high set $H_t$ or the low set $L_t$.\footnote{We use a notation to describe \TruVar that differs slightly from the one used in \cite{bogunovic-nips16a}.}
\TruVar proceeds in \emph{epochs} $i$, which often last for multiple iterations, and which have a corresponding target confidence bound $\eta_{(i)}$.
This value shrinks according to $\eta_{(i+1)}=r\eta_{(i)}$ as epochs advance, where the reduction factor $r\in(0, 1)$ is a configuration parameter.
It chooses which candidate points $\lambda$ to evaluate based on the summed reduction of the current uncertainty estimate $\beta_{(i)}^{1/2}\sigma(\evalpointbar)$ across all unclassified candidates $\evalpointbar\in U_t$, where we clip the reduction at uncertainty  $\eta_{(i)}$.
The algorithm allows this reduction to be scaled by a (user-defined) cost function $\operatorname{cost}(\lambda)$, and $\beta_{(i)}$ is an epoch-dependent scaling factor.
The epoch advances once all points in $U_t$ have confidence below $(1+\deltabar)\eta_{(i)}$, with configurable slack parameter $\deltabar>0$.

Bogunovic et al.~\cite{bogunovic-nips16a} present two variants of their algorithm: One for LSE, as explained above,
in which unclassified points $U_t$ are categorized into $H_t$ and $L_t$, and one for optimization, where points are classified based on whether they are potential risk minimizers.
\TruVarImp extends \TruVar for LSE with an implicitly defined threshold.
Its underlying idea is to simultaneously track a set of potential minimizers $M_t$ and a set of points $U_t$ that are unclassified with respect to the implicit threshold, thereby using both variants of \TruVar in one.
One approach to tracking minimizers and level sets simultaneously was presented by Gotovos et al.~\cite{gotovos2013}; however, our algorithm selects points to evaluate in a manner similar to \TruVar, based on which point reduces the posterior confidence to a given threshold the most.

\subsection{Our TruVaRImp Algorithm}

Following Bogunovic et al.~\cite{bogunovic-nips16a}, we model the objective function $\objfun(\evalpoint):\candset\to\R$ as a GP with constant mean $\mu_0=0$, kernel function $k(\evalpoint,\evalpoint{}')$ and, in general, heteroscedastic measurement noise $\epsilon(\evalpoint)\sim \normal(0, \sigma^2(\evalpoint))$, where we observe noisy objective values $\objobs=\objfun(\evalpoint)+\epsilon(\evalpoint)$.
We use the kernel matrix $\Kmat_t = \bigl[k(\evalpoint_i, \evalpoint_j)\bigr]_{i,j=1}^t$ for points $\{\evalpoint_i\}_{i=1}^t$ observed until after iteration $t$, the white noise term $\mathbf{\Sigma}_t = \diag(\sigma^2(\evalpoint_1),\ldots,\sigma^2(\evalpoint_t))$, and the cross-covariance vector $\kv_t(\evalpoint) = [k(\evalpoint_i, \evalpoint)]_{i=1}^t$.
The posterior distribution of $\objfun(\evalpoint_{t+1})$ given the vector of the first $t$ observations $\mathbf{\objobs}_{1:t}$ is thus $\normal(\mu(\evalpoint_{t+1}), \sigma^2(\evalpoint_{t+1}))$, where
\begin{align}
    \mu_t(\evalpoint_{t+1}) &= \kv_t(\evalpoint_{t+1})^\top(\Kmat_t+\mathbf{\Sigma}_t)^{-1}\mathbf{\objobs}_{1:t}\label{eq:mu_update}\\
    \sigma_t^2(\evalpoint_{t+1}) &= k(\evalpoint_{t+1}, \evalpoint_{t+1}) -  \kv_t(\evalpoint_{t+1})^\top(\Kmat_t+\mathbf{\Sigma}_t)^{-1}\kv_t(\evalpoint_{t+1})\mathrm{.}\label{eq:sigma_update}
\end{align}
Similar to Bogunovic et al.~\cite{bogunovic-nips16a}, for some $\evalpointbar\in\candset$ we denote the posterior variance of $\objfun(\evalpointbar)$ given observations at $\evalpoint_1,\ldots,\evalpoint_{t-1}$ as well as an additional observation at $\evalpoint$ as $\sigma_{t-1|\evalpoint}^2(\evalpointbar)$.
This represents a one-step lookahead evaluation of the posterior variance:
Given that the algorithm has already made $t-1$ observations, this quantity indicates how evaluating a specific choice of $\objfun(\evalpoint)$ will influence the posterior variance at another point $\evalpointbar$.
Notably, $\sigma_{t-1|\evalpoint}^2(\evalpointbar)$ does not depend on the observed value $\objobs$ at $\evalpoint$ and can be calculated efficiently, as described in Appendix~B of Bogunovic et al.~\cite{bogunovic-nips16a}.

Algorithm~\ref{alg:truvarimp} shows our \TruVarImp algorithm.
It chooses where to evaluate the objective, so that points from a finite input set $\candset$ are efficiently classified as above or below a threshold level $h$.
Without loss of generality, we consider the case of minimization, i.e., finding a \CashomonSet{} with a threshold defined in terms of the (unknown) minimum of the objective function: $\objfun_{\mathrm{min}} = \min_{\lambda \in \candset} \objfun(\lambda)$.
We set $h = \objfun_{\mathrm{min}} \times (1 + \varepsilon_{\mathrm{rel}}) + \varepsilon_{\mathrm{abs}}$, where we allow the expression of the \CashomonSet{} in terms of relative ($\varepsilon_{\mathrm{rel}} \ge 0$; assuming $\forall \evalpoint: \objfun(\evalpoint) \ge 0$) or absolute ($\varepsilon_{\mathrm{abs}} \ge 0$) offset from the minimum.
Typically, one of these is zero; we include both variables in this formula to remain general.

\TruVarImp keeps track of a set of potential minimizers $M_t$ of $\objfun(\cdot)$, a set of candidate points $L_t$ classified as below the threshold, a set of points classified as above the threshold $H_t$, and a set of points unclassified with respect to the threshold $U_t$.
$L_t$, $U_t$ and $H_t$ form a partition of $\candset$, and $M_t \subseteq U_t \cup L_t$.
For each point, the algorithm calculates confidence intervals
\begin{align}
[l_t(\evalpoint), u_t(\evalpoint)] =[\mu_t(\evalpoint) - \beta_{(i)}^{1/2}\sigma_t(\evalpoint),
\mu_t(\evalpoint) + \beta_{(i)}^{1/2}\sigma_t(\evalpoint)].\label{eq:confint}
\end{align}
It works under the assumption that, if $\beta_{(i)}$ is chosen large enough, the true $c(\lambda)$ is in $[l_t(\evalpoint), u_t(\evalpoint)]$ with high probability.
A point $\evalpoint$ is therefore removed from $M_t$ whenever its confidence interval does not overlap with the pessimistic estimate of the minimum $\objfun_{\mathrm{min},t}^\mathrm{pes}=\min_{\evalpointbar \in M_{t-1}} u_t(\evalpointbar)$, and a point is removed from $U_t$ and added to either $L_t$ or $H_t$ whenever its confidence interval does not overlap with the range of possible threshold values $[h_t^\mathrm{opt}, h_t^\mathrm{pes}]$ (see Algorithm~\ref{alg:truvarimp}, lines \ref{alg:hopt}--\ref{alg:hpes}).
Like \TruVar, \TruVarImp proceeds in epochs, which count up whenever confidence intervals for all remaining candidates in $U_t$ and $M_t$ fall below a threshold proportional to $\eta_{(i)}$ (Algorithm~\ref{alg:truvarimp}, line \ref{alg:epoch}).
Figure~\ref{fig:truvarimp_illustration} shows the main components of the algorithm at iteration $t$.
Our theoretical analysis (Appendix~\ref{sec:analysis}) shows that, with suitably chosen algorithm parameters, this approach identifies points in the \CashomonSet{} up to a given accuracy $\epsilon$ with high probability.

\begin{algorithm}[H]
\algrenewcommand\algorithmicrequire{\textbf{Input:}}
\algrenewcommand\algorithmicensure {\textbf{Output:}}
    \caption{Truncated Variance Reduction for Implicitly Defined LSE (\TruVarImp)} \label{alg:truvarimp}
    \begin{algorithmic}[1]
        \Require Objective function $c(\cdot)$; Domain $\candset$; evaluation-cost function $\operatorname{cost}(\evalpoint)$; GP prior ($\mu_0$, $\sigma(\evalpoint)$; $\kxxp$); confidence bounding parameters $\deltabar > 0$, $r \in (0,1)$, $\{\beta_{(i)}\}_{i \ge 1}$, $\eta_{(1)} > 0$; $\varepsilon_{\mathrm{rel}} \ge 0$, $\varepsilon_{\mathrm{abs}} \ge 0$
        \Ensure Sets $(L, H, U)$ predicted to be respectively below, above and unclassified with respect to the implicit threshold $h=\min_{\lambda \in \candset} \objfun(\lambda) \times(1+\varepsilon_{\mathrm{rel}})+\varepsilon_{\mathrm{abs}}$
        \State Initialize $L_0=H_0=\emptyset$, $M_0=U_0=\candset$, epoch number $i=1$
        \For {$t = 1,2,\dotsc$}
        \State Letting $\Delta_i(D, \sigma^2, p) := \sum\nolimits_{\evalpointbar \in D}\max\left\{ p^2 \beta_{(i)}\,\sigma^2(\evalpointbar) - \eta_{(i)}^2, 0 \right\}\mathrm{,}$
        \State \label{alg:acquisition} \vspace{-0.5cm} \[
            \begin{split}
                \hspace*{-1.5ex}  \text{find } \evalpoint_t& \leftarrow \argmax_{\evalpoint \in \candset} \Big\{
                \Big[\Delta_i(U_{t-1}, \sigma_{t-1}^2, 1) - \Delta_i(U_{t-1}, \sigma_{t-1|\evalpoint}^2, 1) +\\
                & \Delta_i(M_{t-1}, \sigma_{t-1}^2, 1{+}\varepsilon_{\mathrm{rel}}) - \Delta_i(M_{t-1}, \sigma_{t-1|\evalpoint}^2, 1{+}\varepsilon_{\mathrm{rel}})\Big] \Big/ \operatorname{cost}(\evalpoint)\Big\}\\
            \end{split}
        \]
        \State Evaluate objective at $\evalpoint_t$, observe noisy $\objobs_t$.
        \State Calculate $\mu_t$ and $\sigma_t$ according to \eqref{eq:mu_update}--\eqref{eq:sigma_update}, and
        $l_t(\evalpoint)$ and $u_t(\evalpoint)$ according to \eqref{eq:confint}
        \State $\objfun_{\mathrm{min},t}^\mathrm{pes} \leftarrow \min_{\evalpointbar \in M_{t-1}} u_t(\evalpointbar)$\label{alg:cpes}
        \State $h_t^\mathrm{pes} \leftarrow \objfun_{\mathrm{min},t}^\mathrm{pes} \times (1 + \varepsilon_{\mathrm{rel}}) + \varepsilon_{\mathrm{abs}}$\label{alg:hpes}
        \State $h_t^\mathrm{opt} \leftarrow \min_{\evalpointbar \in M_{t-1}} l_t(\evalpointbar) \times (1 + \varepsilon_{\mathrm{rel}}) + \varepsilon_{\mathrm{abs}}$\label{alg:hopt}
        \State $H_t \leftarrow H_{t-1}$, $L_t \leftarrow L_{t-1}$, $U_t \leftarrow \emptyset$, $M_t\leftarrow \emptyset$

        \For {each $\evalpoint \in U_{t-1}$}\label{alg:uloop}
            \If {$u_t(\evalpoint) \le h_t^\mathrm{opt}$}
                \State $L_t \leftarrow L_t \cup \{\evalpoint\}$
            \ElsIf {$l_t(\evalpoint) > h_t^\mathrm{pes}$}
                \State $H_t \leftarrow H_t \cup \{\evalpoint\}$
            \Else
                \State $U_t \leftarrow U_t \cup \{\evalpoint\}$
            \EndIf
        \EndFor
        \For {each $\evalpoint \in M_{t-1}$}\label{alg:mloop}
            \If {$l_t(\evalpoint) \le c_{\mathrm{min},t}^\mathrm{pes}$}
                \State $M_t \leftarrow M_t \cup \{\evalpoint\}$
            \EndIf
        \EndFor

        \While{$\max_{\evalpoint \in U_t} \beta_{(i)}^{1/2}\sigma_{t}(\evalpoint) \le (1+\deltabar)\, \eta_{(i)}$ and $\max_{\evalpoint \in M_t} \beta_{(i)}^{1/2}\sigma_{t}(\evalpoint) \le \frac{1+\deltabar}{1 + \varepsilon_{\mathrm{rel}}} \eta_{(i)}$}\label{alg:epoch}
            \State $i \leftarrow i+1$, $\eta_{(i)} \leftarrow r \times \eta_{(i-1)}$.\label{alg:epoch2}
        \EndWhile
        \If{budget is exhausted}
            \State \Return $H \gets H_t, L \gets L_t, U \gets U_t$
        \EndIf
        \EndFor
    \end{algorithmic}
\end{algorithm}
\paragraph{Implementation Notes.}
In practice, the GP's kernel HPs are usually unknown and estimated from observed performance data.
In applications such as HPO, they are re-fit in each iteration, so a different kernel function is used at each step.
We calculate $\sigma_{t-1|\evalpoint}^2(\evalpointbar)$ from the kernel function at $t-1$ to make use of the efficiency gains mentioned above, but it ceases to be the exact value of $\sigma_t^2(\lambda')$ for $\lambda=\lambda_t$.
This is a common simplification; as seen in Section~\ref{sec_experiments}, the algorithm works well in practice.

Furthermore, the CASH hyperparameter space $\Lam_{\text{CASH}}$ is generally hierarchical: the HPs of inducer $\ind_{m}$ are active only when model class $m$ is selected.
Although one could design problem-specific kernels that transfer information across model classes, we make the problem-agnostic assumption that performance observations are informative only within the same model class.
This yields the block-diagonal kernel $k(\lambda_m,\lambda'_{m'})=\delta_{m m'}k_m(\lambda_m,\lambda'_m)$ with independent class-specific kernels $k_m(\cdot,\cdot)$ defined on each subspace $\Lam_{m}$.

\begin{figure}[thb]
    \centering
    \includegraphics[width=1\textwidth]{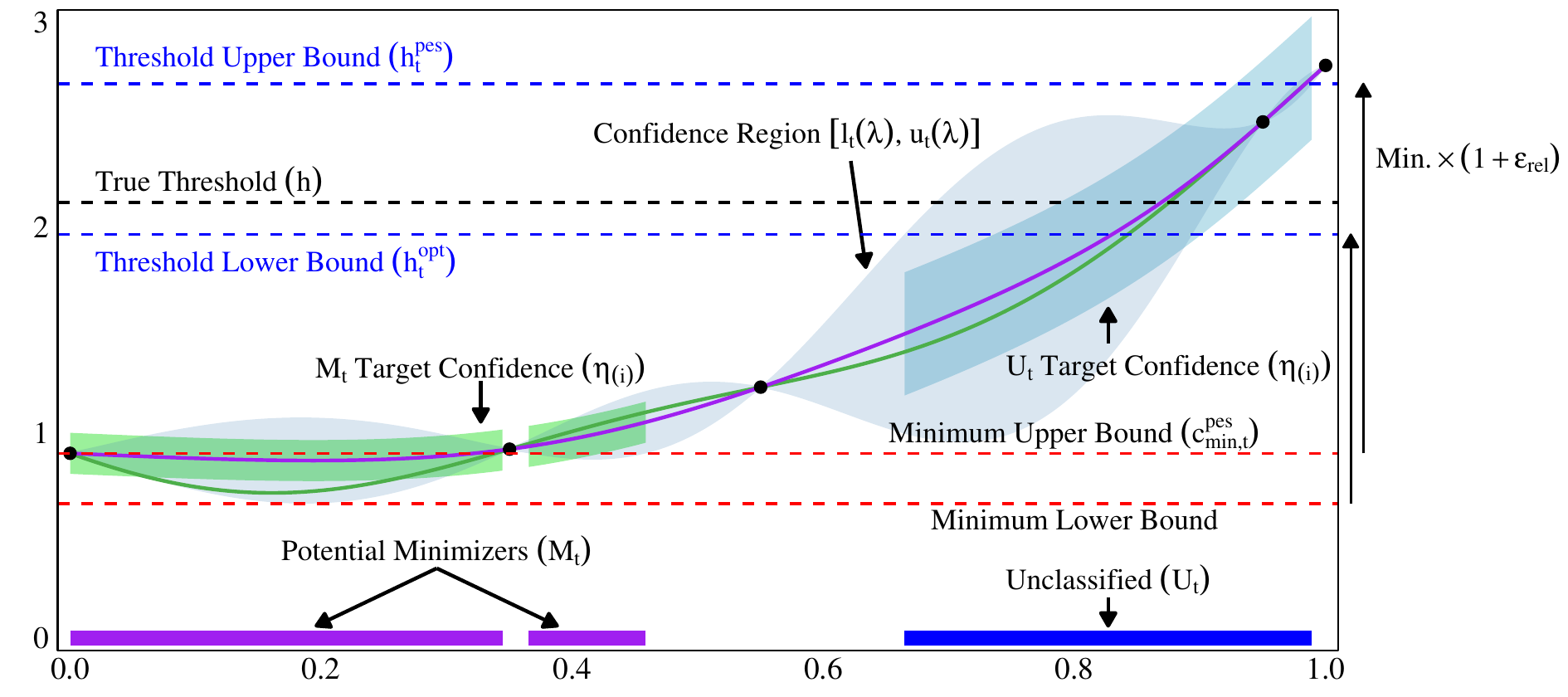}
    \caption{Illustration of the \TruVarImp algorithm.
    The true (unknown) function is shown in green, with known function values $c(\lambda)$ as black dots, through which a GP model is fit (purple line with transparent confidence region).
    $M_t$ is the set of potential minimizers (purple strip at the bottom), for which the confidence interval crosses the ``pessimistic'' (i.e., highest possible within confidence region) minimum $c_{\mathrm{min},t}^\mathrm{pes}$ (upper red dashed line).
    Configuration points $\lambda$ are classified as belonging to the lower or upper set ($L_t$ and $H_t$, not shown), or remain unclassified ($U_t$, blue strip at the bottom), depending on whether their confidence interval crosses the region of plausible threshold values $[h_t^\mathrm{opt}, h_t^\mathrm{pes}]$ (blue dashed lines).
    In this illustration, these are relative to the minimum with $\varepsilon_{\mathrm{rel}} = 2, \varepsilon_{\mathrm{abs}} = 0$.\\
    \TruVarImp selects points based on how much evaluating them will reduce the variance for candidates in both $M_t$ and $U_t$ that goes beyond their target confidence values $\eta_{(i)}$ the most (green and purple ribbons around the GP mean).
    }
    \label{fig:truvarimp_illustration}
\end{figure}

\section{Experiments} \label{sec_experiments}

We
\begin{inparaenum}[\bf (1)]
    \item evaluate the reliability and efficiency of \TruVarImp, and\label{list_exp1}
    \item empirically examine practical implications of extending \RashomonSets{} to the CASH hypothesis space.\label{list_exp2}
\end{inparaenum}
To address~(\ref{list_exp1}), we compare our proposed \TruVarImp algorithm against LSE baselines, HPO, and random evaluation in Section~\ref{sec:lse_experiment}, addressing our first two research questions (RQs):
\begin{enumerate}[leftmargin=*,label=\textbf{RQ~1.\arabic*}, nosep]
    \item Can \TruVarImp find the \CashomonSet{} accurately and efficiently? \label{rq_setsize}
    \item How does \TruVarImp perform in comparison to baselines? \label{rq_baselines}
\end{enumerate}
Regarding~(\ref{list_exp2}), we explore the models in the \CashomonSets{} found by \TruVarImp, considering the following questions in Section~\ref{sec:fi_for_RE}:
\begin{enumerate}[leftmargin=*,label=\textbf{RQ~2.\arabic*}, nosep]
    \item Do Rashomon and \Cashomon{} models differ in predictive performance?\label{rq_pred_perf}
    \item Do Rashomon and \CashomonSets{} differ in predictive multiplicity?\label{rq_pred_mult}
    \item How do FI values vary across Rashomon and \Cashomon{} models?\label{rq_feature_imp}
\end{enumerate}
The code and reproduction scripts are available at:\\ \url{https://github.com/slds-lmu/paper_2024_rashomon_set}.

\subsection{Experimental Setup} \label{sec_exp_setup}

To study \CashomonSets{} in a large, diverse search space, we use model classes that are commonly used in practice and that also vary in interpretability.
We evaluate DTs fitted using the CART algorithm (\texttt{cart}) and GOSDT (\texttt{gosdt}). We also use gradient boosting models (\texttt{xgb}), feedforward NNs (\texttt{nnet}), support vector machines (\texttt{svm}) with two kernels (\texttt{linear} and \texttt{radial}), and elastic nets (\texttt{glmnet}); the latter as an LM for regression and a logistic model for classification datasets.
The HP search spaces are mostly lifted from \cite{binder-automl20a} and are listed in Appendix~\ref{app:learners}.

We focus on datasets used by %
Xin et al.~\cite{xin-neurips22a} and their proposed binarization for comparability.
We consider various classification tasks: \emph{COMPAS} (\texttt{CS}, with and without binarization), \emph{German Credit} (\texttt{GC}), and binarized versions of \emph{car evaluation} (\texttt{CR}), \emph{Monk~2} (\texttt{MK}), \emph{breast cancer} (\texttt{BC}), and \emph{FICO} (\texttt{FC}).
We also include two regression tasks: \emph{Bike Sharing} (\texttt{BS}) and a synthetic dataset (\texttt{ST}).
See Appendix~\ref{app:datasets} for more details on the datasets.

We split each dataset into two parts.
We use 2/3 for model selection, where we estimate $\riske$ using a $10\times{}10$-fold repeated cross-validation, and then refit final models on the full training split.
We use the remaining 1/3 of the data as test data for performance, predictive multiplicity, and FI analysis.
To compute \CashomonSets{}, we sample 8000 HPCs uniformly at random from the search space of each learning algorithm as the candidate set $\candset$.
We evaluate each HPC on each dataset, measuring the Brier score for classification and the root mean squared error (RMSE) for regression tasks.
\texttt{gosdt} can handle only binary data and is therefore evaluated only on our binarized classification datasets.
For a given task, we consider the best HPC found over all considered model classes as the reference model $\fref$ and use its performance to determine the \CashomonSet{} cutoff using commonly used~\cite{mueller-ecml23a, cavus-isf26a} values $\varepsilon_{\mathrm{rel}} = 0.05$ and $\varepsilon_{\mathrm{abs}}=0$.

\subsection{Experiments on Level Set Estimation}\label{sec:lse_experiment}
We evaluate the ability of \TruVarImp to identify the \CashomonSet{}, treating the LSE problem as an active learning problem:
configuration points are selected and evaluated so that a surrogate model fitted to the observed design points can be used to classify whether unobserved configurations belong to the \CashomonSet{}.
This viewpoint is reflected in the acquisition strategy of \TruVarImp, which prioritizes evaluations that reduce uncertainty about the membership of currently unclassified points.
It is also aligned with the goal stated at the end of Section~\ref{sec_cashomonsets}, namely to sample from the conditional distribution $P_{\text{CASH}}(\lambda\mid \ind(\D,\lambda) \in \Cset)$, if the candidate set is itself sampled from $P_{\text{CASH}}$ and the surrogate predicts performance values reasonably well.

There are different ways to use a GP surrogate to predict the \CashomonSet{}.
A fully Bayesian evaluation could average \CashomonSet{} membership over posterior GP samples:
In each GP sample, the reference configuration and cutoff value would, in general, differ, and the uncertainty about the true best model would be effectively marginalized over.
However, this would not reflect actual practice, where researchers care about generating alternative models and explanations for a concrete, fitted reference model.
Therefore, we take an application-centric view, evaluating a \TruVarImp run based on the \CashomonSet{} it would generate from the best \emph{observed} model.
We run \TruVarImp with objective $\hat{c}(\lambda)=\riske(\ind(\D,\lambda))$.
After $t$ evaluations, we use the incumbent $f_{\text{ref}}^{(t)}=\ind(\D,\hat{\lambda}_t^{*})$ with $\hat{\lambda}_t^{*}\in \argmin_{\lambda\in\{\lambda_1,\cdots,\lambda_t\}}\hat{c}(\lambda)$ as reference model.
The corresponding cutoff $\hat{h}_t=(1+\varepsilon_{\mathrm{rel}})\riske(f_{\text{ref}}^{(t)})+\varepsilon_{\mathrm{abs}}$ is then used to classify configurations into the predicted \CashomonSet{} based on model predictions, $\hat{\Cset}_{t}=\{\lambda\in\candset\mid\mu_{t}(\lambda)\le\hat{h}_t\}$.
We compare this against the (finite) ground truth sample from the (infinite) \CashomonSet{}, calculated as a subset of $\candset$ using the precomputed $\hat{c}(\lambda)$ with the minimum over $\candset$ as reference, as this is the representative sample from the \CashomonSet{} that a practitioner would obtain by evaluating all samples in $\candset$.

We compare \TruVarImp against the following LSE algorithms:
\texttt{STRADDLE}~\cite{bryan2005active}, \texttt{LSE} and $\texttt{LSE}_\texttt{IMP}$ (the level set estimation algorithm and its implicit LSE variant introduced by Gotovos et al.~\cite{gotovos2013}), and the original \TruVar algorithm~\cite{bogunovic-nips16a}.
We furthermore consider \texttt{RANDOM} (random configuration point evaluations), and \texttt{OPTIMIZE}: HPO via Bayesian optimization, using the expected improvement acquisition function~\cite{garnett-book23a}.
By comparing our LSE-based approach with a Bayesian optimization approach on the same CASH space, we implicitly compare against Cavus et al.~\cite{cavus-isf26a}, who extract their \RashomonSet{} from the archive of a standard AutoML system (H2O AutoML~\cite{ledell-automl20a}) without explicitly targeting the underlying LSE definition or evaluating what they lose by disregarding this (they effectively take the best models from the underlying random search of H2O AutoML).
See Appendix~\ref{appendix_truvarimp_config} for more details on algorithm configuration parameters.

We use a candidate set of 8000 pre-computed performance values per model class; the CASH space thus has 48,000 candidates in total for binary datasets, 40,000 for the others.
Figures~\ref{fig:cash_results} and~\ref{fig:cash_results_app} (Appendix~\ref{appendix_truvarimp_config}) show progress in surrogate model quality, which we quantify using the F1 score, since \CashomonSet{} membership is a highly imbalanced binary classification problem.
On average, \TruVarImp outperforms baselines in classifying \CashomonSet{} members, enabling accurate threshold decisions. %
Notably, the \CashomonSet{} task appears difficult for baselines, likely because the search space is complex and multimodal.
\begin{figure}[ht]
  \centering
  \includegraphics[width=\linewidth, clip]{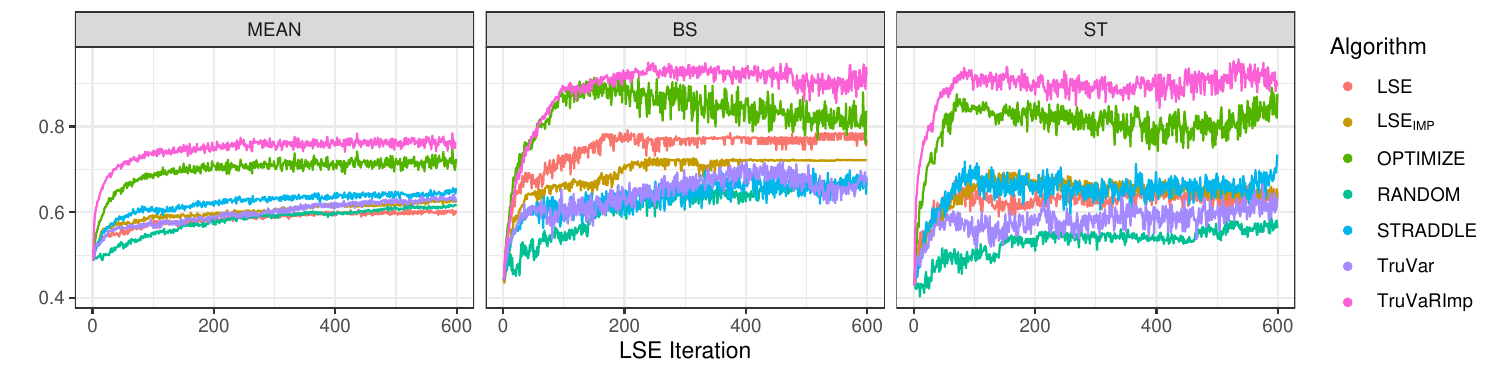}
  \caption{Mean \CashomonSet{} algorithm progress in terms of F1 score of the surrogate model predicting \CashomonSet{} membership.
  The iteration count does not include the initial sample of 30 points per model class.
  The left facet shows the mean values over all datasets.
  The middle and right show individual performance on two example datasets.
  }
  \label{fig:cash_results}
\end{figure}

\subsection{An Application Perspective}\label{sec:fi_for_RE}

As we are not aware of any algorithm producing \RashomonSets{} across diverse model classes, we compare our \CashomonSets{} with \RashomonSets{} found by a state-of-the-art \RashomonSet{} algorithm, namely \TreeFARMS \cite{xin-neurips22a}, which is a popular, ready-to-use method for finding \RashomonSets{} for GOSDT \cite{lin-icml20a}.
We illustrate that performance, predictive multiplicity, and FI values may differ between models from the \TreeFARMS \RashomonSet{} and the \CashomonSets{}.
For the latter, we use the results from running \TruVarImp on the candidate set, i.e., HPCs actually evaluated by \TruVarImp, and fit the resulting configurations on the training split of each dataset.

\paragraph{Predictive Performance.}
The top row of Figure~\ref{fig_pred_perf_RC} shows the Brier score distributions on the test set for three representative tasks (\texttt{BC}, \texttt{CR}, and binarized \texttt{CS}).
The reference models of Rashomon and \CashomonSets{} differ notably, since models from the model class \texttt{gosdt} often perform comparably worse than some of the other model classes.
Only the \CashomonSet{} of \texttt{CS} contains \texttt{gosdt} models; for \texttt{BC} and \texttt{CR}, tree-based models do not perform well enough.\footnote{Table~\ref{tab_exp_model_CS}, Appendix~\ref{appendix_exp_results}, shows the number of models per model class contained in the \CashomonSets{} per dataset.}
Hence, models in our \CashomonSets{} perform, on average, better than those found by \TreeFARMS.
For better reference models, $\varepsilon_{\mathrm{rel}}\,\risk(\fref)$ becomes smaller, and performance tends to scatter less across models, even on a test set, as seen for \texttt{BC} and \texttt{CR}.

\begin{figure}[tb]
    \centering

    \begin{subfigure}[b]{0.32\textwidth}
        \centering
        \includegraphics[width=\textwidth]{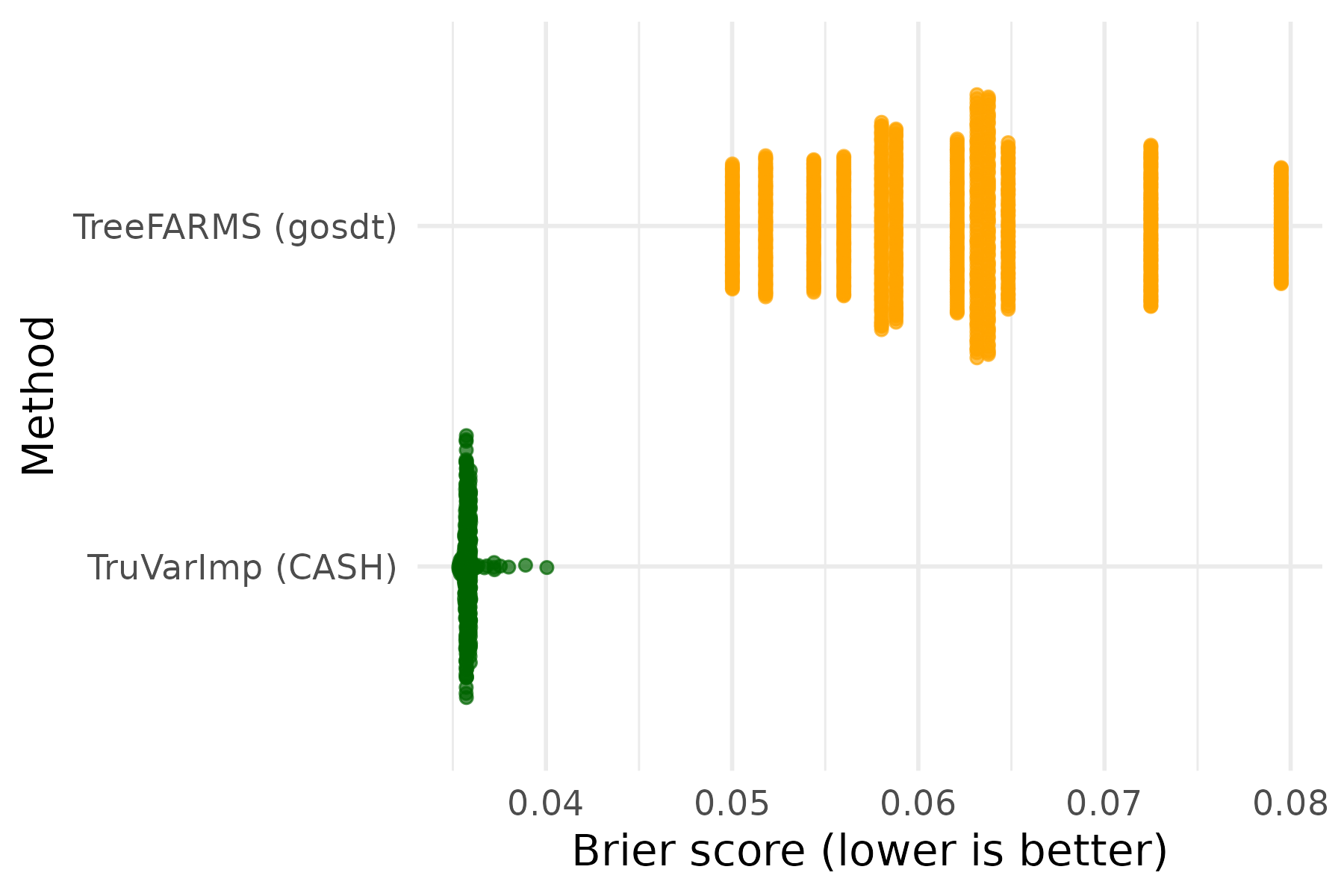}
        \label{fig:bc}
    \end{subfigure}
    \hfill
    \begin{subfigure}[b]{0.32\textwidth}
        \centering
        \includegraphics[width=\textwidth]{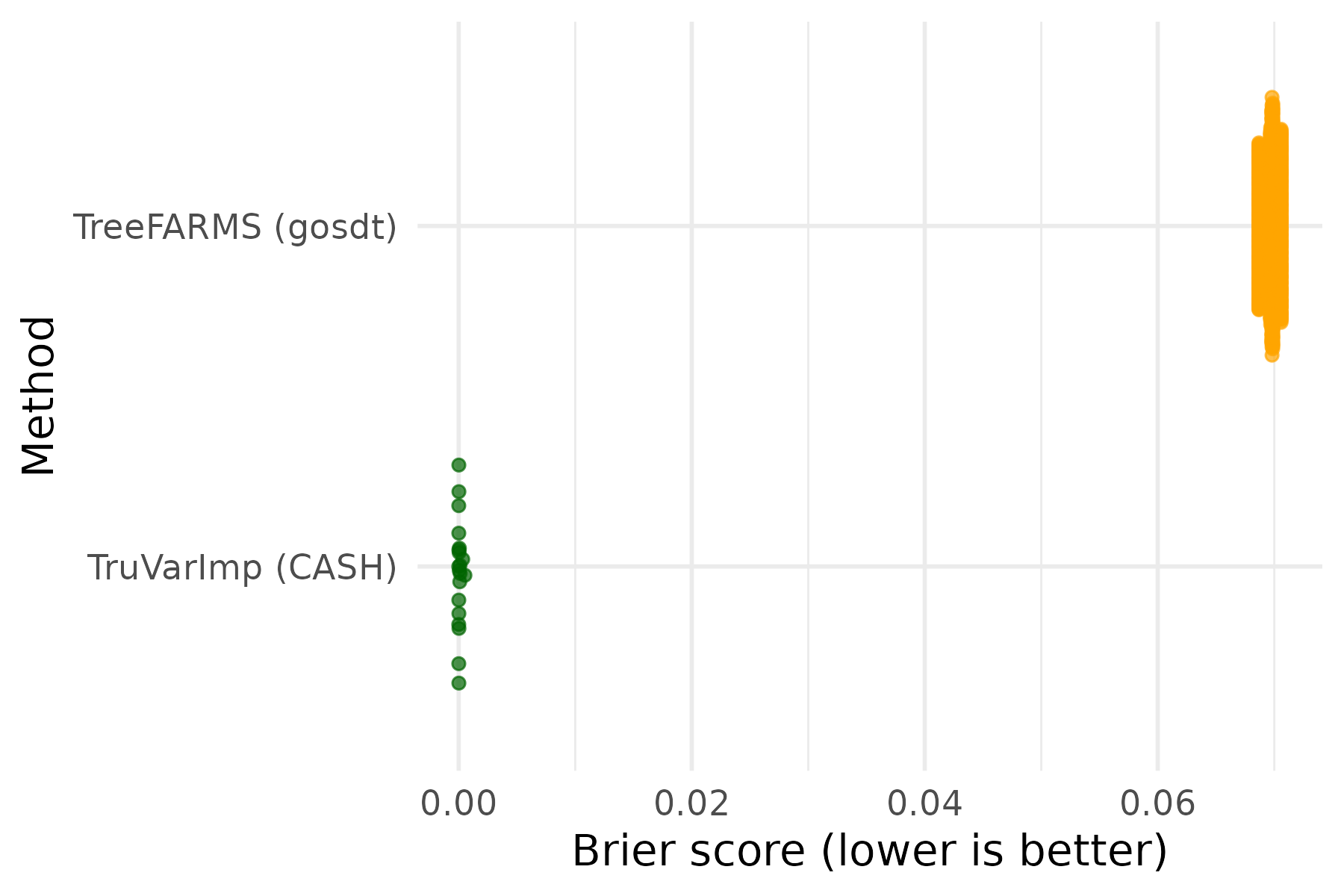}
        \label{fig:cr}
    \end{subfigure}
    \hfill
    \begin{subfigure}[b]{0.32\textwidth}
        \centering
        \includegraphics[width=\textwidth]{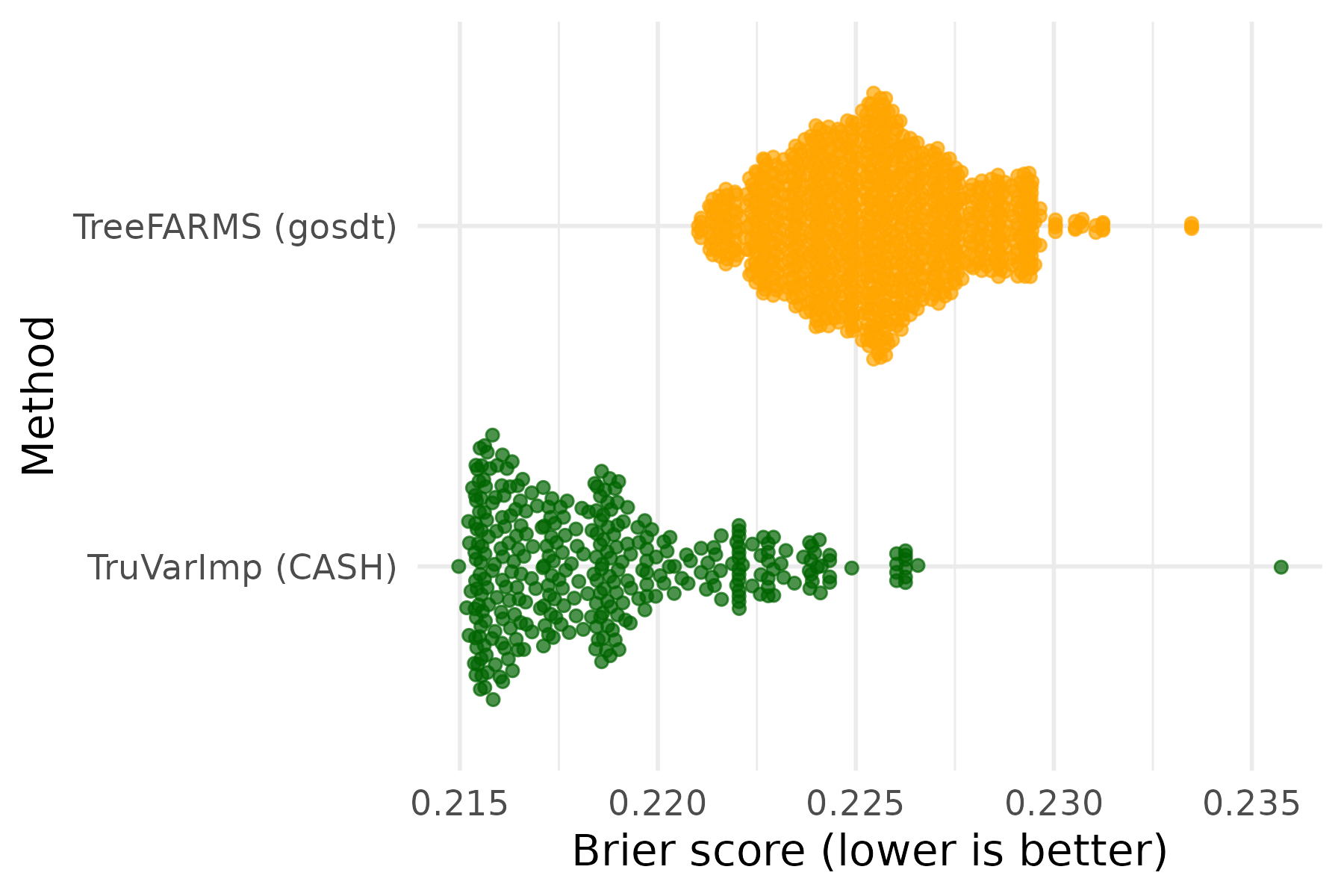}
        \label{fig:csbin}
    \end{subfigure}

    \vspace{1em}

    \begin{subfigure}[b]{0.32\textwidth}
        \centering
        \includegraphics[width=\textwidth, trim={0 0 0 1.7cm}, clip]{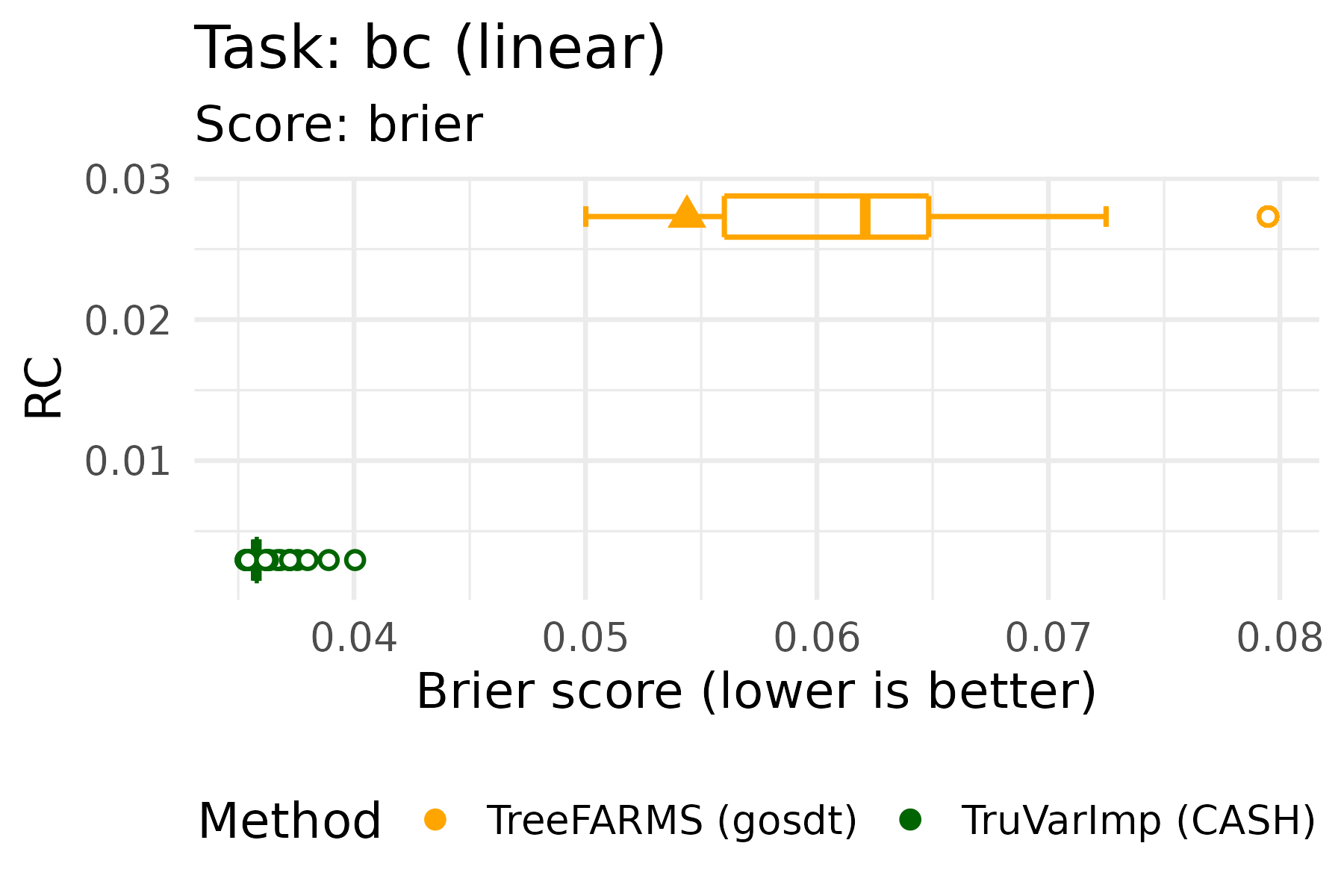}
        \caption{Task \texttt{BC}}
        \label{fig:bc-RC_vs_bestperf}
    \end{subfigure}
    \hfill
    \begin{subfigure}[b]{0.32\textwidth}
        \centering
        \includegraphics[width=\textwidth, trim={0 0 0 1.7cm}, clip]{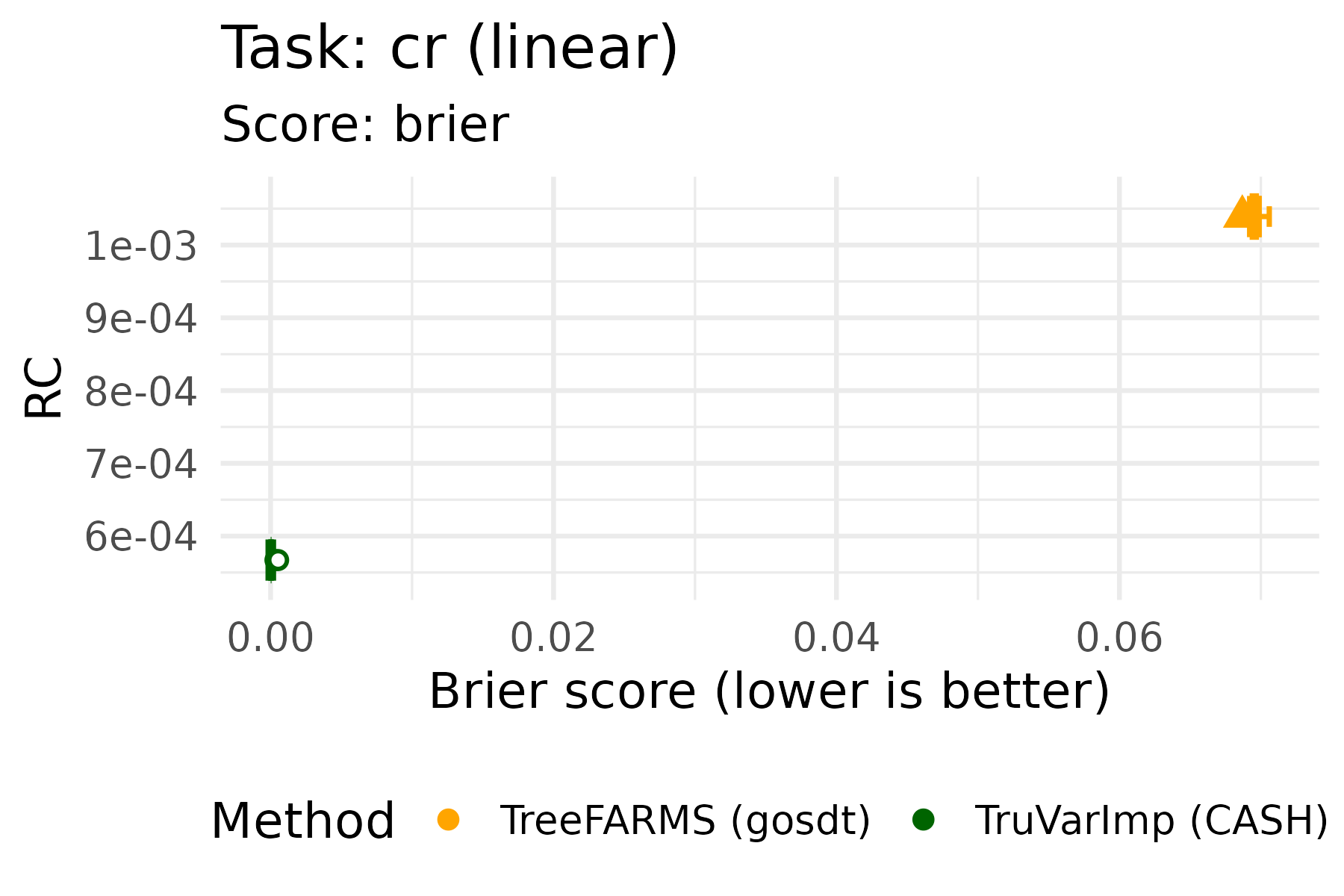}
        \caption{Task \texttt{CR}}
        \label{fig:cr-RC_vs_bestperf}
    \end{subfigure}
    \hfill
    \begin{subfigure}[b]{0.32\textwidth}
        \centering
        \includegraphics[width=\textwidth, trim={0 0 0 1.7cm}, clip]{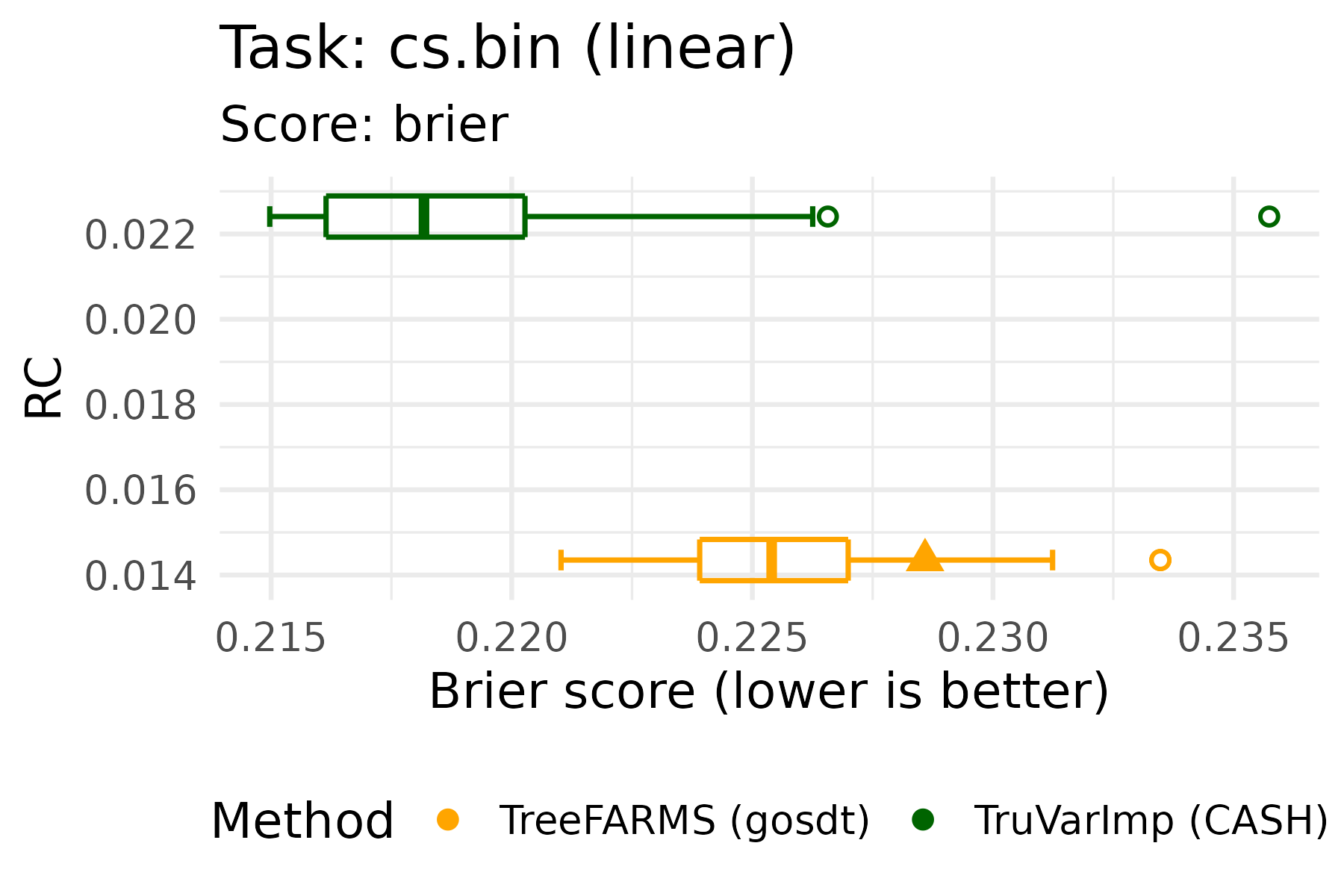}
        \caption{Task \texttt{CS} (binarized)}
        \label{fig:csbin-RC_vs_bestperf}
    \end{subfigure}

    \caption{Test performance and RC for tasks \texttt{BC, CR,} and \texttt{CS} (binarized) across the \TreeFARMS \RashomonSet{} and our \CashomonSet{}. Top row: Brier score distributions. Bottom row: RC (y-axis) versus Brier score (x-axis); horizontal boxplots summarize test score distributions; triangles indicate the reference model.}
    \label{fig_pred_perf_RC}
\end{figure}

\paragraph{Predictive Multiplicity.}
To analyze predictive multiplicity, we consider \emph{Rashomon capacity} (RC, \cite{hsu-neurips22a,hsu_arxiv24}), which measures the spread in predictions across models within a \RashomonSet{}.\footnote{We slightly adapt the original per-observation, classification-specific definition to measure predictive multiplicity across the entire dataset, and extend it to regression problems. The formal definition and a pseudocode for RC are in Appendix~\ref{appendix_RC}.}
It is non-trivial to trade off or compare RC versus predictive performance in evaluation.
On the one hand, since our goal is to assess predictive multiplicity of near-optimal models, we aim to find Rashomon or \CashomonSets{} with higher RC while achieving good predictive performance.
On the other hand, we expect that model sets containing less well-generalizing models have a higher RC.
What we can certainly say is this:
\begin{inparaenum}[\bf (1)]
    \item Not detecting what the optimal prediction model and its performance are, seems problematic. \label{enumlabela}
    \item Assuming (\ref{enumlabela}) is well-approximated, we want to know how much variability exists across the space of all admissible models.
\end{inparaenum}

The bottom row of Figure~\ref{fig_pred_perf_RC} plots the RC against the test performances for the \TreeFARMS \RashomonSet{} and our \CashomonSet{}.
The triangles indicate the reference model in each set.
Across all considered datasets, there is a tendency for the \TreeFARMS \RashomonSet{} to yield higher RC values than our \CashomonSets{}, however, at the cost of a higher Brier score.
We also observe that the reference models in the \TreeFARMS \RashomonSets{} perform worse on the test data than the ones in our \CashomonSets{}, especially when comparing the reference model's test performance to the remaining models within each set.
Notably, in task \texttt{CS}, our \CashomonSet{} not only contains better generalizing models than the \TreeFARMS \RashomonSet{}, but also results in a higher RC.

\paragraph{Feature Importance.}
Since similarly well-performing models may rely differently on features, we study the FI values induced by models in the \CashomonSet{} using \emph{variable importance clouds} (VICs, \cite{dong-nmi20a}), i.e., sets of FI vectors, one for each model in the set.
We quantify model reliance via permutation feature importance (PFI; \cite{breiman-mlj01a,fisher-jmlr19a}), estimated by averaging ten permutation runs and scaled such that they preserve the within-model ranking and relative importance of features while enabling comparisons across heterogeneous models and model classes.\footnote{For formal definitions of VIC and PFI see Appendix~\ref{appendix_vic_results}.}
\begin{figure}[tb]
    \centering
    \includegraphics[clip, trim=0cm 0.2cm 0cm 1.1cm, width = \linewidth]{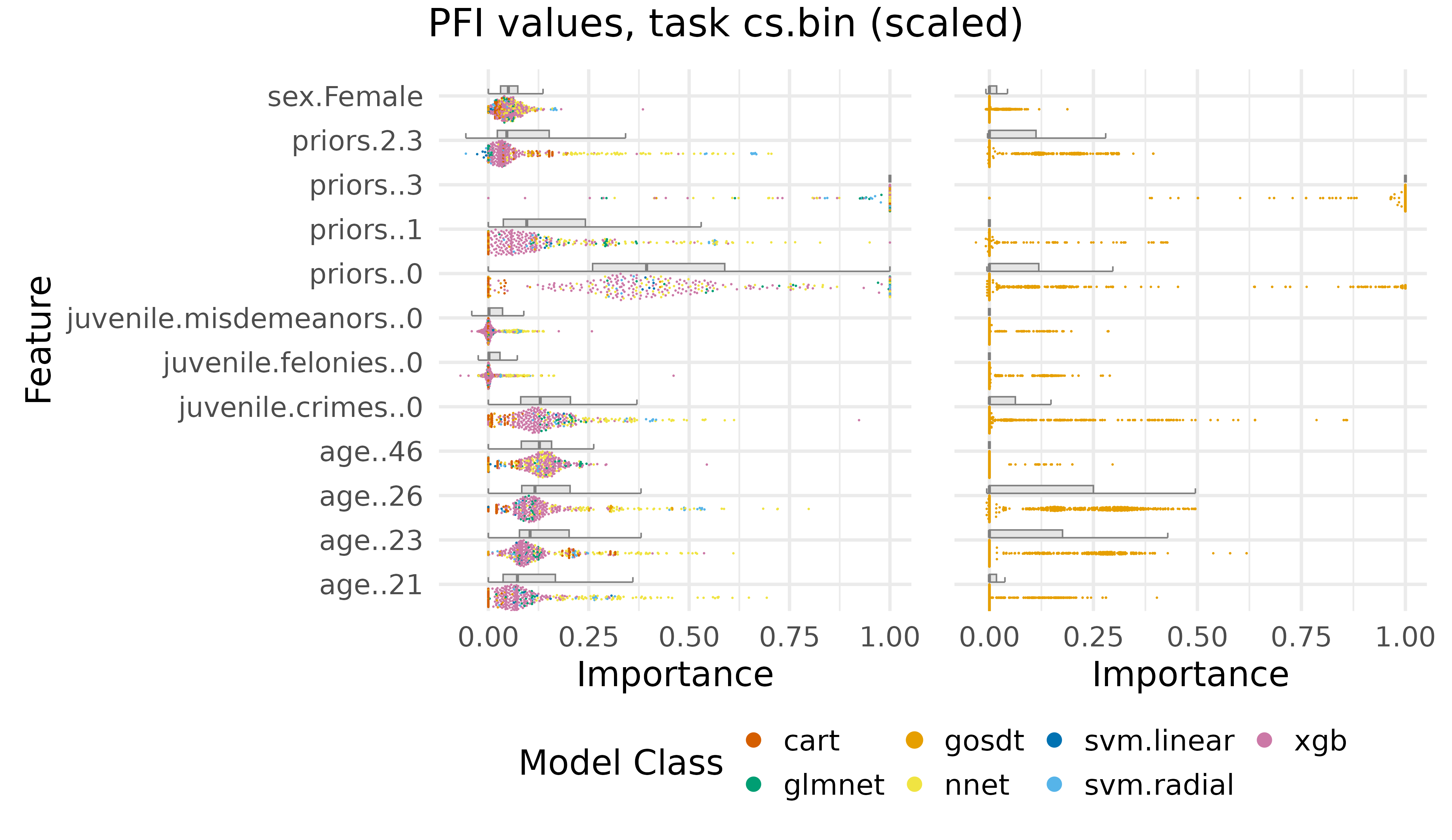}
    \caption{VICs for our \CashomonSet{} (left) and for the \TreeFARMS \RashomonSet{} (right) on task \texttt{CS}. For each feature, PFI values across models are shown as a boxplot together with a point cloud colored by model class, where identical values are vertically jittered. PFI values are scaled: the maximal PFI value for a model equals 1. In our \CashomonSet{}, most models are \texttt{xgb} models (178), followed by \texttt{nnet} (66) and \texttt{cart} (33) models, but all model classes are present.}
    \label{fig_vic}
\end{figure}
Figure~\ref{fig_vic} shows the VICs for the binarized \texttt{CS} dataset: the left panel (colored by model class) shows our \CashomonSet{}, the right shows the \TreeFARMS \RashomonSet{}.
In the right panel, variation across \texttt{gosdt} models already produces a notable spread in FI values, e.g., for \texttt{priors..3} and \texttt{priors..0}, the assigned importance varies across the models from \emph{not important} ($0$) to \emph{most important} ($1$).
The left panel reveals a second, between-class layer of variation since FI values scatter differently across model classes.
Moreover, the two model sets disagree markedly on individual features: e.g., for \texttt{priors..0}, most models in the \RashomonSet{} assign near-zero or near-one importance, whereas models in the \CashomonSet{} show a bulk in $[0.25,0.6]$.
This disagreement reveals that the choice of model class can substantially change which features the models rely on most; a pattern we observe across all datasets considered, though the degree of variation differs by task (see Appendix~\ref{appendix_vic_results}).
To the best of our knowledge, a further interpretation of FI results in a Rashomon context remains an open research area beyond their utility for model selection.

\section{Discussion, Limitations and Future Work} \label{sec_conclusion}

With the \CashomonSet{}, we present the first formal extension of the \RashomonSet{} across multiple model classes. 
Finding \CashomonSets{} can naturally be framed as an implicit level set estimation problem, for which we introduce the novel algorithm \TruVarImp and provide convergence guarantees.
In contrast to prior work, our approach is model-class-independent, unifies different model classes within a single \RashomonSet{} framework, and supports efficient search over the resulting joint hypothesis space.
Our experiments show that \TruVarImp efficiently identifies HPCs from the \CashomonSet{}. %

We investigate various properties of \CashomonSets{}, including predictive multiplicity using the Rashomon capacity and FI using VICs. 
Our analyses reveal a possible trade-off between predictive performance and predictive multiplicity, showing that \RashomonSet{} analyses can be misleading when the set contains poor predictors. 
This highlights the importance of constructing \RashomonSets{} with strong predictive performance, a goal that may be harder to achieve when the search is restricted to a single model class.
Moreover, FI values can differ substantially across model classes, implying that explanations derived from a single-class \RashomonSet{} may provide only a partial picture: they may miss important explanation variability and may change, once additional, equally competitive model classes are taken into account.

Despite these contributions, several limitations remain and motivate future work.
We have not investigated how the $\varepsilon_{\mathrm{abs}}$ and $\varepsilon_{\mathrm{rel}}$ affect the \CashomonSet{} composition; for very small cutoff values, the \CashomonSet{} is more likely to only contain a single model class, and there may be novel effects, e.g., higher RC, at larger cutoff values where many different model classes enter the \CashomonSet{}.
In particular, in certain scenarios, it might be preferable to find inherently interpretable models. 
Currently, this is only possible by restricting \TruVarImp by pre-selecting certain model classes. 
Future work could investigate other ways of biasing preferentially evaluating more desirable model classes, or generally a more diverse set of model classes. 
Finally, we have only started to explore applications of \CashomonSets{}.
The framework allows new analyses of predictive multiplicity, model selection under ambiguity, and robustness of FI scores across heterogeneous model classes.

\end{bibunit}
\clearpage
\appendix
\begin{bibunit}[splncs04]

\section*{Declaration of GenAI usage}
We used generative AI tools for improving the text, methodology, and code. Concretely, we used the following tools for the following workloads:
\begin{enumerate}
    \item text: we used ChatGPT, Codex CLI and Claude Code to suggest writing improvements and shortening text.
    \item methodology: we used Claude Code and Codex-CLI to suggest improvements to the \TruVarImp R implementation.
    The algorithm itself was human-authored.
    ChatGPT and Gemini were involved in generating the algorithm and implementation for our version of the Rashomon Capacity.
    \item code: we used Claude code, Codex, ChatGPT, and Gemini to generate, improve, and review experiment code.
\end{enumerate}
We have manually checked all generated output, and take full responsibility for the content.

\section{Illustrative Example} \label{appendix_RS_HPO_viz}

To illustrate the relationship between \RashomonSets{} and models obtained via HP variation, we consider a simple regression problem.
Data are drawn from the data-generating process $\yv=X_1+X_2+\epsilon$ and modeled using a linear model with covariates $X_1$ and $X_2$ and without intercept.
Figure~\ref{fig_example_ridge_regression} shows, for one specific random sample of size 30, the true-risk \RashomonSet{}, the empirical-risk \RashomonSet{}, and the set of models obtained by fitting an elastic-net model over HP values.
The overlap indicates the part of the empirical \RashomonSet{} achievable through HP variation, $$\Rset(\varepsilon=0.15,\fref=\hat{f^{\star}},\Hspace^{\D}_\text{HPO}).$$
Data was drawn from the distributions $$\epsilon\sim\normal(0, 0.5^2) \text{ and } \begin{pmatrix}X_1\\X_2\end{pmatrix}\sim\normal\left(\begin{pmatrix}0\\0\end{pmatrix},\begin{pmatrix}1&2\\2&5\end{pmatrix}\right)$$.

\begin{figure}[H]
    \centering
      \includegraphics[width=.8\linewidth]{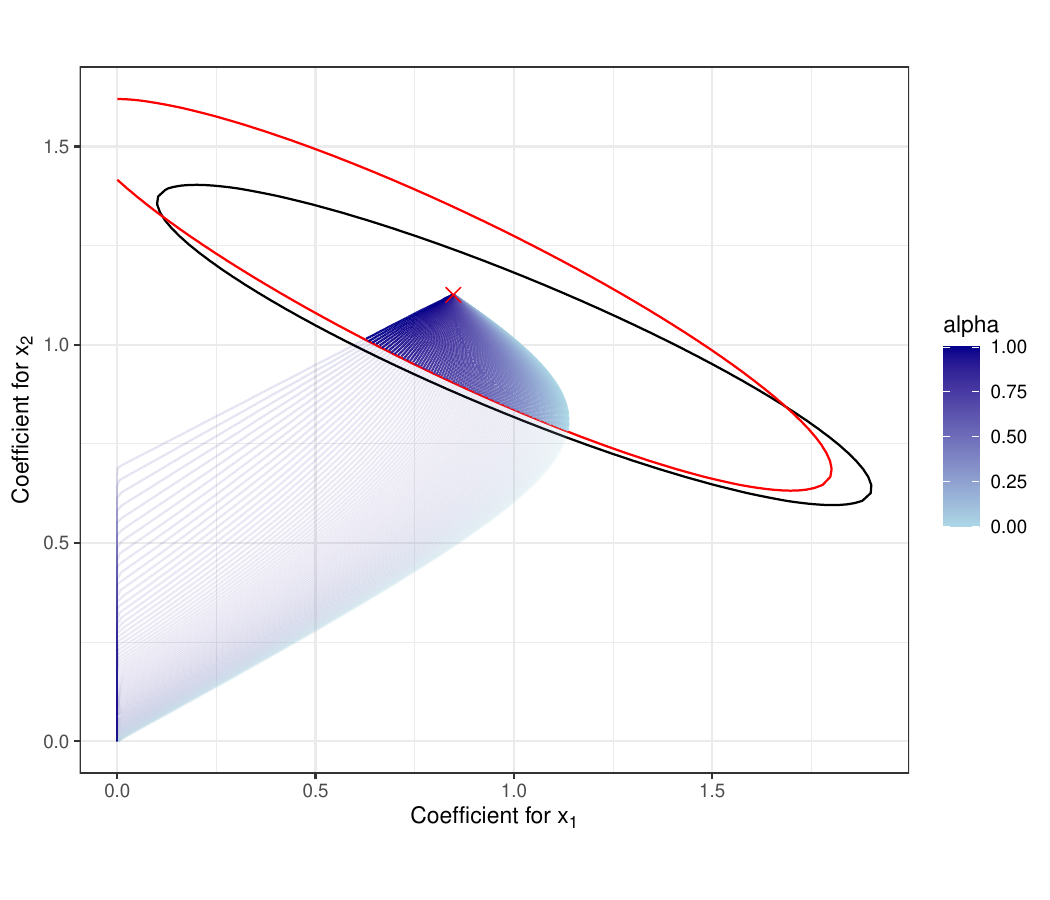}
    \caption{Visualization of \RashomonSets{} on a simple regression problem.
    Every coordinate on the plot corresponds to a linear model with given coefficients.
    The \enquote{true} model $f^{\star}$ lies at $(1, 1)$ and has RMSE $0.5$.
    The black ellipse indicates the \RashomonSet{} when taking $\fref=f^{\star}$, $\varepsilon=0.15$, and the true risk $\risk$.
    The red cross indicates $\hat{f^{\star}}$, a model that minimizes the empirical risk $\riske$, and the red ellipse represents the \RashomonSet{} when using this model as reference.
    The shaded blue area represents $\Hspace^{\D}_\text{HPO}$ as defined in Eq.~\eqref{eq_F_HP}, the models resulting from fitting an elastic net model (using the \texttt{glmnet} R-package \cite{zou-jrsssb05a}) and parameterized with the \texttt{lambda} and \texttt{alpha} HPs, on the specific available dataset.
    The shade of blue indicates the value of \texttt{alpha}; \texttt{lambda} is not shown (though models with larger \texttt{lambda} values lie closer to the origin, corresponding to the regularizing effect of \texttt{lambda}).
    The shaded area within the red ellipse is the intersection of the \RashomonSet{} and the models achievable through HP variation.
    }
    \label{fig_example_ridge_regression}
\end{figure} 
\section{Datasets}\label{app:datasets}

We consider the Bike Sharing (\texttt{BS}) dataset \cite{fanaee-pai14}, predicting the hourly count of rented bikes of the Capital Bike Sharing system in Washington, D.C.; the COMPAS (\texttt{CS}) dataset \cite{angwin2016machine}, predicting recidivism of defendants; and the German Credit (\texttt{GC}) dataset \cite{data-gc}, predicting good vs.\ bad credit risks.

Furthermore, we consider various binarized datasets, as the implementation of GOSDT only works on such ~\cite{lin-icml20a}. 
Concretely, we use the following datasets as provided in the experiments GitHub repository of Xin et al.~\cite{xin-neurips22a}: car evaluation database (\texttt{CR}, \cite{cardatasett}), \emph{Monk~2} (\texttt{MK}, \cite{monk2dataset}), \emph{COMPAS} (binarized), and \emph{FICO} (\texttt{FC})\footnote{There is no published reference for this dataset, also, the official competition website is no longer online. 
We refer to \href{https://openml.org/d/45553}{https://openml.org/d/45553} for further details, and obtain the dataset from OpenML~\cite{bischl-pattern25a}\label{footnote-fico}.}. 
Furthermore, use the \emph{breast cancer} (\texttt{BC}, \cite{breastcancerdataset}) dataset,  manually binarized as described by Xin et al.~\cite{xin-neurips22a}.

Finally, we analyze a synthetically generated dataset (\texttt{ST}) proposed by Ewald et al.~\cite{Ewald-xAI24a}, for which concrete expectations regarding PFI values can be formulated.
\texttt{ST} consists of five features $X_1, \dots, X_5$ and target~$Y$. 
$X_1, X_3$ and $X_5$ are i.i.d. $N(0,1)$; $X_2$ and $X_4$ are (noisy) copies of $X_1$ and $X_3$, respectively: $X_2~:=~X_1~+~\epsilon_2$, $\epsilon_2~\sim~N(0, 0.001)$; $X_4~:=~X_3~+~\epsilon_4$, $\epsilon_4~\sim~N(0, 0.1)$.
The target only depends on $X_4$ and $X_5$ via linear effects and a bivariate interaction:
\begin{align}\label{eq_st_dgp}
    Y := X_4 + X_5 + X_4 \times X_5 + \epsilon_Y,\, \epsilon_Y \sim N(0, 0.1).
\end{align}
We evaluated \texttt{BS} and \texttt{ST} as regression tasks, and \texttt{CS} and \texttt{GC} as binary classification tasks, using the RMSE and the Brier score \cite{brier1950verification} as the performance objectives, respectively. 
More information about the datasets is provided in Table~\ref{tbl:datasets}.

\begin{table}[H]
\centering
\caption{Overview of the datasets used in the experiments. We abbreviate the columns as \#Num for number of numeric features, \#Cat for number of categorical features, and $N$ for total number of samples.}
\label{tbl:datasets}
\begin{tabular}{l l r r r l l}
\toprule
Name & Description & $N$ & \#Num & \#Cat & Type & Reference \\
\midrule
\texttt{BC} & Breast Cancer dataset & 683 & 9 & 0 & classification & \cite{breastcancerdataset} \\
\texttt{CR} & Car dataset & 1728 & 15 & 0 & classification & \cite{cardatasett} \\
\texttt{CS} & COMPAS dataset & 6172 & 5 & 6 & classification & \cite{angwin2016machine} \\
\texttt{CS} (binarized) & binarized \texttt{CS} & 6907 & 12 & 0 & classification & \cite{angwin2016machine} \\
\texttt{FC} & FICO dataset & 10459 & 17 & 0 & classification & footnote \ref{footnote-fico} \\
\texttt{GC} & German Credit dataset & 1000 & 3 & 18 & classification & \cite{data-gc} \\
\texttt{MK} & Monk 2 dataset & 169 & 11 & 0 & classification & \cite{monk2dataset} \\
\texttt{BS} & Bike Sharing dataset & 17379 & 9 & 4 & regression & \cite{fanaee-pai14} \\
\texttt{ST} & Synthetic dataset & 10000 & 5 & 0 & regression & \cite{Ewald-xAI24a} \\
\bottomrule
\end{tabular}
\end{table}

\section{Model Classes and HP Spaces}
\label{app:learners}

We used six different model classes in our experiments to cover a wide range of model types. 
Below, we list each model class, the software (R) packages used, and any preprocessing performed. 
The HP search space for each learning algorithm is provided in separate tables.
The search spaces roughly follow Binder et al.~\cite{binder-automl20a}, with some HP ranges expanded slightly to avoid too many cases of optima lying out of the search constraints, and others reduced if they are known to have little effect on performance or if their evaluation has an outsized effect on compute cost.
\enquote{(log)} indicates the HP was sampled on a log-transformed scale.
HPs other than those tuned or specified in the description are set to their default values.
Preprocessing was done using the \texttt{mlr3pipelines} \cite{mlr3pipelines} package.

\paragraph{XGBoost (\texttt{xgb}).}
Using the \texttt{xgboost} package \cite{chen-kdd16a}.
Factorial features were treatment-encoded.

\begin{table}[H]
\centering
\caption{HP search space for \texttt{xgb}.}
\begin{tabular}{lll}
\hline
HP & Type & Range \\
\hline
\texttt{alpha} & continuous & [\(10^{-3}, 10^3\)] (log) \\
\texttt{colsample\_bylevel} & continuous & [0.1, 1] \\
\texttt{colsample\_bytree} & continuous & [0.1, 1] \\
\texttt{eta} & continuous & [\(10^{-4}, 1\)] (log) \\
\texttt{lambda} & continuous & [\(10^{-3}, 10^3\)] (log) \\
\texttt{max\_depth} & integer & [1, 20] (log) \\
\texttt{nrounds} & integer & [1, 5000] (log) \\
\texttt{subsample} & continuous & [0.1, 1] \\
\hline
\end{tabular}
\label{tbl:xgb_hparams}
\end{table}

\paragraph{Decision Tree (\texttt{cart}).}
Using the \texttt{rpart} package \cite{Therneau1999}.
\begin{table}[H]
\centering
\caption{HP search space for \texttt{cart}.}
\label{tbl:tree_hparams}
\begin{tabular}{lll}
\hline
HP & Type & Range \\
\hline
\texttt{cp} & continuous & [\(10^{-4}, 0.2\)] (log) \\
\texttt{minbucket} & integer & [1, 64] (log) \\
\texttt{minsplit} & integer & [2, 128] (log) \\
\hline
\end{tabular}
\end{table}

\paragraph{Neural Network (\texttt{nnet}).}
Implemented via the \texttt{nnet} package \cite{Ripley2009}, with \texttt{maxit} set to 5000 and \texttt{MaxNWts} set to $10^6$.
Continuous features are preprocessed by centering and scaling to unit variance.
\begin{table}[H]
\centering
\caption{HP search space for \texttt{nnet}.}
\label{tbl:nnet_hparams}
\begin{tabular}{lll}
\hline
HP & Type & Range \\
\hline
\texttt{decay} & continuous & [\(10^{-6}, 1\)] (log) \\
\texttt{size} & integer & [8, 512] (log) \\
\texttt{skip} & categorical & \{TRUE, FALSE\} \\
\hline
\end{tabular}
\end{table}

\paragraph{Elastic Net (\texttt{glmnet}).}
Using the \texttt{glmnet} package \cite{zou-jrsssb05a}. 
Factorial features were treatment-encoded.
\begin{table}[H]
\centering
\caption{HP search space for \texttt{glmnet}.}
\label{tbl:glmnet_hparams}
\begin{tabular}{lll}
\hline
HP & Type & Range \\
\hline
\texttt{alpha} & continuous & [0, 1] \\
\texttt{lambda} & continuous & [\(10^{-4}, 10^3\)] (log) \\
\hline
\end{tabular}
\end{table}

\paragraph{Support Vector Machine (\texttt{svm}).}
Using the \texttt{e1071} package \cite{Meyer1999}, with \texttt{tolerance} set to $10^{-4}$ and setting \texttt{type} \texttt{"eps-regression"} for regression and \texttt{"C-classification"} for classification.
Factorial features were treatment-encoded, and constant features were removed if they were encountered in any cross-validation fold.
\begin{table}[H]
\centering
\caption{HP search space for \texttt{svm}.}
\label{tbl:svm_hparams}
\begin{tabular}{lll}
\hline
HP & Type & Range \\
\hline
\texttt{cost} & continuous & [\(10^{-4}, 10^4\)] (log) \\
\texttt{gamma} & continuous & [\(10^{-4}, 10^4\)] (log) \\
\texttt{kernel} & categorical & \{linear, radial\} \\
\hline
\end{tabular}
\end{table}

\paragraph{Generalized Optimal Sparse Decision Tree (\texttt{gosdt}).}
Using the Python \texttt{gosdt} package \cite{lin-icml20a}, with \texttt{allow\_small\_reg = TRUE} to allow \texttt{regularization} values below $1/n$.
We use a wrapper around the original Python package to extend it for probability predictions.  
\begin{table}[H]
\centering
\caption{HP search space for \texttt{gosdt}.}
\label{tbl:gosdt_hparams}
\begin{tabular}{lll}
\hline
HP & Type & Range \\
\hline
\texttt{regularization} & continuous & [\(10^{-4}, 0.49\)] (log) \\
\texttt{balance} & categorical & \{TRUE, FALSE\} \\
\texttt{depth\_budget} & integer & [1, 100] (log) \\
\hline
\end{tabular}
\end{table} 

\section{\TruVarImp Algorithm and Analysis}\label{sec:analysis}

\newcommand{\cc}{\operatorname{cost}}
\newcommand{\actualgreekxi}{\text{\usefont{OML}{cmm}{m}{it}\char24}}  %

Our theoretical analysis generally follows Bogunovic et al.~\cite{bogunovic-nips16a} with slightly different notation and the necessary adjustments for the modified acquisition function.

Central to the theory behind both \TruVar and \TruVarImp is the following probabilistic event $\mathcal{E}$:
\begin{equation}
    \label{eq:good-confidence-event}
    \mathcal{E} := \left\{ c(\lambda) \in [l_t(\lambda), u_t(\lambda)] \text{ for all $t$ and, given that $t$, for all } \lambda \in U_{t-1} \cup M_{t-1} \right\}.
\end{equation}
The aim is to choose $\beta_{(i)}$ large enough for a desired probability $\mathds{P}(\mathcal{E})$.
This builds upon Lemma~5.1 from~\cite{srinivas2012information}, which applies a union bound over all points in a candidate set $D$ and all time steps $t$, and makes use of the fact that the normal distribution has exponentially decaying tails, as well as $\sum_{t\ge 1} t^{-2} = \pi^2/6$. The underlying assumption is that the objective $c(\cdot)$ is drawn from the GP and observed with i.i.d.\ Gaussian noise. We state the lemma as:
\begin{lemma}[Srinivas et al., 2012]
    \label{lemma:srinivas2012information}
    For a given $\delta \in (0,1)$ and a candidate set $D$, let $\beta_t \ge 2 \log \frac{|D|t^2\pi^2}{6\delta}$. Then with probability at least $1-\delta$, the true $c(\lambda)$ is in $[ \mu_t(\lambda) - \beta_t^{1/2}\sigma_t(\lambda), \mu_t(\lambda) + \beta_t^{1/2}\sigma_t(\lambda)]$ for all $\lambda \in D$.
\end{lemma}

The following Theorem~\ref{thm:truvarimp-analysis} works by choosing $\beta_{(i)}$ that fulfills this condition based on a sequence $C_{(i)}$ that is a deterministic (but objective-function dependent) upper bound on the total cost spent during each epoch $i$.
This also provides an upper bound on the total cost spent until a given level of accuracy is achieved.

Algorithm~\ref{alg:truvarimp} iteratively refines $M_t$ so that (under $\mathcal{E}$) it always contains the true minimizer of $\objfun(\cdot)$.
Based on the resulting possible range of minima $[\min_{\lambda \in M_{t-1}} l_t(\lambda), \max_{\lambda \in M_{t-1}} u_t(\lambda)]$, it uses $\varepsilon_{\mathrm{rel}}$ and $\varepsilon_{\mathrm{abs}}$ to calculate bounds for $h$ such that $h\in [h_t^\mathrm{opt}, h_t^\mathrm{pes}]$.
$H_t$ and $L_t$ are updated with all points for which the confidence interval does not overlap with $[h_t^\mathrm{opt}, h_t^\mathrm{pes}]$.

Like \TruVar, \TruVarImp proceeds in epochs where the confidence of points in unclassified sets is refined, but since the uncertainty about the minimum has an effect of size $(1 + \varepsilon_{\mathrm{rel}})$ larger than the uncertainty about the objective function values in $U_t$, the uncertainty is upweighted by this factor both in the acquisition function (Line~\ref{alg:acquisition}) and in the epoch progression condition (Line~\ref{alg:epoch}).

\subsection{Setup and Definitions}
We first show that $c_{\mathrm{min}}^{\mathrm{pes}}$, $h_t^\mathrm{pes}$ and $h_t^\mathrm{opt}$ are good bounds for the minimum and LSE threshold, respectively. 
\begin{lemma}\label{lemma:boxing}
    Let $\lambda^{*}\in\argmin_{\lambda\in\candset}c(\lambda)$.
    Under event $\mathcal{E}$, $\lambda^{*}$ stays in $M_t$, and $h_t^\mathrm{pes} \ge h \ge h_t^\mathrm{opt}$ for all $t$.
    No point $\lambda$ s.t.\ $c(\lambda)>h$ enters $L_t$, no point $\lambda$ s.t.\ $c(\lambda)<h$ enters $H_t$.
\end{lemma}
\begin{proof}[by induction]
\begin{enumerate}
    \item Since $M_0=\candset$, $\lambda^{*}$ is clearly in $M_0$.
    \item Assume that $\lambda^{*}\in M_{t-1}$. Then
    \[
        \min_{\lambda\in M_{t-1}} l_t(\lambda)
        \le l_t(\lambda^{*})
        \le c(\lambda^{*})
        = c_{\mathrm{min}}.
    \]
    Since $\varepsilon_{\mathrm{rel}}\ge 0$, multiplying by
    $1+\varepsilon_{\mathrm{rel}}$ and adding $\varepsilon_{\mathrm{abs}}$
    preserves order, so
    \begin{align*}
        \min_{\lambda\in M_{t-1}} l_t(\lambda)\times
        \bigl(1+\varepsilon_{\mathrm{rel}}\bigr)
        + \varepsilon_{\mathrm{abs}}
        &\le
        c_{\mathrm{min}}\times
        \bigl(1+\varepsilon_{\mathrm{rel}}\bigr)
        + \varepsilon_{\mathrm{abs}},
    \end{align*}
    and hence $h_t^{\mathrm{opt}}\le h$.

    Likewise, for every $\lambda\in M_{t-1}$,
    \[
        u_t(\lambda) \ge c(\lambda)\ge c_{\mathrm{min}},
    \]
    so
    \[
        c_{\mathrm{min},t}^{\mathrm{pes}}
        = \min_{\lambda\in M_{t-1}} u_t(\lambda)
        \ge c_{\mathrm{min}},
    \]
    and the same affine shift yields $h_t^{\mathrm{pes}} \ge h$.

    Finally,
    \[
        l_t(\lambda^{*})\le u_t(\lambda^{*}) = c_{\mathrm{min}} \le c_{\mathrm{min},t}^{\mathrm{pes}}\text{ from above,}
    \]
    so the condition on Line~\ref{alg:mloop} of
    Algorithm~\ref{alg:truvarimp} is always met. Hence
    $\lambda^{*} \in M_t$.
\end{enumerate}\qed
\end{proof}
\begin{lemma}
    
\end{lemma}

In order to construct the sequence $C_{(i)}$, we first construct a sequence of ``hard sets'' which are independent of $\lambda_t$ but depend on the (unknown) objective function:
\begin{eqnarray*}
    \overline{U}_{(i)} &:= &\left\{ \lambda \in \candset : |c(\lambda) - h| \le 4(1+\deltabar)\eta_{(i)} \right\},\\
    \overline{M}_{(i)} &:= &\left\{ \lambda \in \candset : c(\lambda) \le \objfun_{\mathrm{min}} + 4\frac{1 + \deltabar}{1+\varepsilon_{\mathrm{rel}}}\eta_{(i)} \right\},
\end{eqnarray*}
and $\overline{U}_{(0)} := \overline{M}_{(0)} := \candset$.

\begin{lemma}\label{lemma:hard-sets}
    For all $t$ after epoch $i$ has ended, and on the event $\mathcal{E}$, we have $U_t \subseteq \overline{U}_{(i)}$ and $M_t \subseteq \overline{M}_{(i)}$.
\end{lemma}
\begin{proof}
Since \TruVarImp enforces that
\[
    \max_{\lambda \in M_t} \beta_{(i)}^{1/2}\sigma_t(\lambda)
    \le
    \frac{1 + \deltabar}{1+\varepsilon_{\mathrm{rel}}}\eta_{(i)}
\]
at the end of epoch $i$ (Algorithm~\ref{alg:truvarimp}, Line~\ref{alg:epoch}), on the event $\mathcal{E}$ the value of $c_{\mathrm{min},t}^{\mathrm{pes}}$ is
\begin{align*}
    c_{\mathrm{min},t}^{\mathrm{pes}}
    &= \min_{\lambda\in M_{t-1}} u_t(\lambda) \\
    &= \min_{\lambda\in M_{t-1}}
       \bigl\{\mu_t(\lambda)+\beta_{(i)}^{1/2}\sigma_t(\lambda)\bigr\} \\
    &\le
       \min_{\lambda\in M_{t-1}}
       \left\{
           l_t(\lambda)
           + 2\frac{1 + \deltabar}{1+\varepsilon_{\mathrm{rel}}}\eta_{(i)}
       \right\} \\
    &= \objfun_{\mathrm{min}}
       + 2\frac{1 + \deltabar}{1+\varepsilon_{\mathrm{rel}}}\eta_{(i)}\text{ (using Lemma~\ref{lemma:boxing}).}
\end{align*}

We also have
\[
    u_t(\lambda)
    = l_t(\lambda) + 2\,\beta_{(i)}^{1/2}\sigma_t(\lambda)
    \le
    l_t(\lambda) + 2\frac{1 + \deltabar}{1+\varepsilon_{\mathrm{rel}}}\eta_{(i)}.
\]
Since $M_t$ contains points for which $l_t(\lambda) \le c_{\mathrm{min},t}^{\mathrm{pes}}$, we have for all $\lambda\in M_t$,
\begin{align*}
    u_t(\lambda)
    &\le
    l_t(\lambda) + 2\frac{1 + \deltabar}{1+\varepsilon_{\mathrm{rel}}}\eta_{(i)} \\
    &\le
    c_{\mathrm{min},t}^{\mathrm{pes}}
    + 2\frac{1 + \deltabar}{1+\varepsilon_{\mathrm{rel}}}\eta_{(i)} \\
    &\le
    c_{\mathrm{min}}
    + 4\frac{1 + \deltabar}{1+\varepsilon_{\mathrm{rel}}}\eta_{(i)},
\end{align*}
which gives
\[
    M_t \subseteq \overline{M}_{(i)}.
\]

Furthermore, on event $\mathcal{E}$,
\begin{align*}
    \bigl|h_t^{\mathrm{opt}} - h\bigr|
    &=
    \left|
        \min_{\evalpointbar \in M_{t-1}} l_t(\evalpointbar)
        - \objfun_{\mathrm{min}}
    \right|
    (1 + \varepsilon_{\mathrm{rel}})
    \le
    2(1 + \deltabar)\eta_{(i)}, \\
    \bigl|h_t^{\mathrm{pes}} - h\bigr|
    &=
    \left|
        c_{\mathrm{min},t}^{\mathrm{pes}} - \objfun_{\mathrm{min}}
    \right|
    (1 + \varepsilon_{\mathrm{rel}})
    \le
    2(1+\deltabar)\eta_{(i)}.
\end{align*}
The points in $U_t$ are the ones for which $c(\lambda)$ is at most
$2(1 + \deltabar)\eta_{(i)}$ away from either $h_t^{\mathrm{opt}}$ or
$h_t^{\mathrm{pes}}$, so
\[
    U_t \subseteq \overline{U}_{(i)}.
\]\qed
\end{proof}

Just like \cite{bogunovic-nips16a}, we now define for a collection of points $S$, possibly with duplicates, the total cost as $\cc(S) := \sum_{\lambda \in S} \cc(\lambda)$, and the posterior variance at $\lambda$ after observing points $\lambda_1, \ldots, \lambda_t$ as well as new points in $S$ as $\sigma_{t|S}^2(\lambda)$.
We also define the minimum cost that would be necessary to jointly reduce the posterior standard deviation on both $U$ and $M$ to respective upper bounds as
\begin{equation}
    C^{*}(\actualgreekxi_U, \actualgreekxi_M; U, M) := \min_{S} \left\{ \cc(S) : \max_{\lambda \in U} \sigma_{0|S}(\lambda) \le \actualgreekxi_U \text{ and } \max_{\lambda \in M} \sigma_{0|S}(\lambda) \le \actualgreekxi_M \right\}\mathrm{.}
\end{equation}
We also make use of the minimum and maximum cost of the evaluation of a single point,
\begin{equation}\label{eq:min-max-cost}
    \cc_{\mathrm{min}} := \min_{\lambda \in \candset} \cc(\lambda)\mathrm{,}
    \quad
    \cc_{\mathrm{max}} := \max_{\lambda \in \candset} \cc(\lambda)\mathrm{.}
\end{equation}
We mostly follow \cite{bogunovic-nips16a} in defining the $\epsilon$-accuracy for the LSE problem\footnote{\cite{bogunovic-nips16a} differs in that they have $c(\lambda)<h$ for $\lambda\in L_t$}:
\begin{definition}
    \label{def:epsilon-accuracy}
    The triplet $(L_t, H_t, U_t)$ is $\epsilon$-accurate if all $\lambda$ in $L_t$ satisfy $c(\lambda) \le h$, all $\lambda$ in $H_t$ satisfy $c(\lambda) > h$, and all $\lambda$ in $U_t$ satisfy $|h - c(\lambda)| \le \frac{\epsilon}{2}$.
\end{definition}

We now have all the ingredients to state our main theorem.

\subsection{Convergence Theorem}\label{appendix_convergence_theorem}

\begin{theorem}
    \label{thm:truvarimp-analysis}
    Assuming a kernel function $k(\lambda, \lambda) \le 1$ for all $\lambda \in \candset$, and for which the variance reduction function $\psi_{t,\lambda}(S) := \sigma_t^2(\lambda) - \sigma_{t|S}^2(\lambda)$ is submodular for any selected points $(\lambda_1, \ldots, \lambda_t)$ and any query point $\lambda \in \candset$.
    For a given $\epsilon > 0$ and $\delta \in (0,1)$, and given values $\left\{C_{(i)}\right\}_{i\ge 1}$ and $\left\{\beta_{(i)}\right\}_{i\ge 1}$ that satisfy
    \begin{eqnarray}
        C_{(i)} &\ge& C^{*}\left(\frac{\eta_{(i)}}{\beta_{(i)}^{1/2}}, \frac{\eta_{(i)}}{\beta_{(i)}^{1/2}(1+\varepsilon_{\mathrm{rel}})}; \overline{U}_{(i-1)}, \overline{M}_{(i-1)}\right) \times\\
        &&\qquad\log\left(\beta_{(i)}\frac{|\overline{U}_{(i-1)}| + (1 + \varepsilon_{\mathrm{rel}})^2|\overline{M}_{(i-1)}|}{\deltabar^2 \eta_{(i)}^2}\right) + \cc_{\mathrm{max}}\mathrm{,}
    \end{eqnarray}
    and
    \begin{equation}\label{eq:beta-i-bound}
        \beta_{(i)} \ge 2 \log \frac{|\candset|(\sum_{i'\le i} C_{(i')})^2 \pi^2 }{6 \delta \cc_{\mathrm{min}}^2}\mathrm{,}
    \end{equation}
    then if \TruVarImp is run until cumulative cost reaches
    \begin{equation}
        C_{\epsilon} = \sum_{i: 8(1+\deltabar)\eta_{(i - 1)} > \epsilon} C_{(i)}\mathrm{,}
    \end{equation}
    with probability at least $1-\delta$, $\epsilon$-accuracy is achieved.
\end{theorem}

\begin{proof}
    We make use of Lemma~\ref{lemma:srinivas2012information} and show that, if the event $\mathcal{E}$ occurs, $C_{(i)}$ bounds the cost spent in epoch $i$ from above.
We define the \emph{scaled excess variance} $\Delta_i(\cdot)$ as
\begin{align}
    \Delta_i(D, \sigma^2, p) := \sum\nolimits_{\evalpointbar \in D}\max\left\{ p^2 \beta_{(i)}\,\sigma^2(\evalpointbar) - \eta_{(i)}^2, 0 \right\}\mathrm{,}
\end{align}
where $D$ is a set of $\lambda$ under evaluation, and create the function
\begin{equation}\label{eq:g-t-definition}
    \begin{aligned}
        g_t(S) & := \Delta_i(U_{t-1}, \sigma_{t-1}^2, 1) - \Delta_i(U_{t-1}, \sigma_{t-1|S}^2, 1)\,+ \\
        & \Delta_i(M_{t-1}, \sigma_{t-1}^2, 1 + \varepsilon_{\mathrm{rel}}) - \Delta_i(M_{t-1}, \sigma_{t-1|S}^2, 1 + \varepsilon_{\mathrm{rel}})\mathrm{,}
    \end{aligned}
\end{equation}
where $S$ is a set of $\lambda$ under evaluation.

The maximum attainable value of $g_t(S)$ over all sets $S$ is
\begin{align}\label{eq:g-t-max-definition}
    g_{t,\mathrm{max}} := \max_{S} g_t(S) = \Delta_i(U_{t-1}, \sigma_{t-1}^2, 1) + \Delta_i(M_{t-1}, \sigma_{t-1}^2, 1 + \varepsilon_{\mathrm{rel}})\mathrm{,}
\end{align}
which arises from the fact that, with arbitrarily many evaluations on a finite set, the posterior variance can be made arbitrarily small.

We construct the submodular covering problem
\begin{equation}\label{eq:submodular-covering-problem}
    \text{minimize}_{S} \cc(S) \quad \text{subject to} \quad g_t(S) = g_{t,\mathrm{max}}\mathrm{.}
\end{equation}
Here, we make use of the fact that our $g_t(S)$ is a linear combination with positive coefficients\footnote{which preserves submodularity~\cite{krause2014submodular}} of Equation~(22) in Bogunovic et al.~\cite{bogunovic-nips16a}, which they show to be submodular, non-decreasing, and $0$ for the empty set.
In Line~\ref{alg:acquisition} \TruVarImp chooses $\lambda_t$ according to the greedy algorithm for the Problem in Equation~\eqref{eq:submodular-covering-problem}, and, again like \cite{bogunovic-nips16a}, we can make use of Lemma~2 of \cite{krause2005note} to state that
\begin{equation}\label{eq:greedy-bound}
    g_t(\{\lambda_t\}) \ge \frac{\cc(\lambda_t)}{\cc(S_t^*)} g_{t,\mathrm{max}}\mathrm{,}
\end{equation}
where $S_t^*$ is the optimal solution to \eqref{eq:submodular-covering-problem}.\footnote{Lemma~2 of \cite{krause2005note} is stated in terms of budgeted submodular maximization. The problem ``$\text{maximize}_{S} g_t(S)$ subject to $\quad\cc(S) \le \cc(S_t^*)$'' has the optimum $g_t(\mathrm{OPT}) = g_{t,\mathrm{max}}$; plugging this into the lemma gives the result in \eqref{eq:greedy-bound}.}
Following Equations~(26)--(29) of \cite{bogunovic-nips16a}, making the same argument with respect to the decreasing nature of $U_t$ and $M_t$, we arrive at
\begin{equation}
    \frac{g_{t+\ell,\mathrm{max}}}{g_{t,\mathrm{max}}} \le \exp\left(-\frac{\sum_{t'=t+1}^{t+\ell} \cc(\lambda_{t'})}{\cc(S_t^*)}\right)\mathrm{,}
\end{equation}
which implies
\begin{equation}
    \cc(S_t^*) \log\left(\frac{g_{t,\mathrm{max}}}{g_{t+\ell,\mathrm{max}}}\right) \ge \sum_{t'=t+1}^{t+\ell} \cc(\lambda_{t'})\mathrm{.}
\end{equation}

We use this inequality to bound the total cost incurred in a given epoch $i$, by noting that, for $t$ at the beginning of an epoch, with $k(\lambda, \lambda) = 1$ and therefore $\sigma_{t-1}^2(\lambda) \le 1$, we have
\begin{equation}
    g_{t,\mathrm{max}} \le \beta_{(i)}\left(|U_{t-1}| + (1 + \varepsilon_{\mathrm{rel}})^2|M_{t-1}|\right) \le \beta_{(i)}(|\overline{U}_{(i-1)}| + (1 + \varepsilon_{\mathrm{rel}})^2|\overline{M}_{(i-1)}|)\mathrm{,}
\end{equation}
where the last inequality assumes the event $\mathcal{E}$ to hold.

If we let $t+\ell$ be the last time step where the same epoch did not yet end, $g_{t+\ell,\mathrm{max}}$ is bounded from below by the fact that there must be one $\lambda$ that is either in $U_{t+\ell}$, for which $\beta_{(i)}\sigma_{t+\ell}^2(\lambda) > (1+\deltabar)^2\eta_{(i)}^2 \ge (1 + \deltabar^2)\eta_{(i)}^2$, or in $M_{t+\ell}$, for which $\beta_{(i)}(1+\varepsilon_{\mathrm{rel}})^2\sigma_{t+\ell}^2(\lambda) > (1+\deltabar)^2\eta_{(i)}^2 \ge (1 + \deltabar^2)\eta_{(i)}^2$.
Plugging either of these into~\eqref{eq:g-t-max-definition} gives
\begin{equation}
    g_{t+\ell,\mathrm{max}} > \deltabar^2\eta_{(i)}^2\mathrm{.}
\end{equation}
To exit epoch $i$, one more evaluation is necessary that will be at most $\cc_{\mathrm{max}}$ in cost.
Denoting $t_{(i)}$ as the time step at the beginning of epoch $i$, we have therefore an upper bound on the cost $C_{(i)}$ spent in epoch $i$ as
\begin{eqnarray*}
    \sum_{t'=t_{(i)}+1}^{t_{(i+1)}} \cc(\lambda_{t'}) & \le & \sum_{t'=t_{(i)}+1}^{t+\ell} \cc(\lambda_{t'}) + \cc_{\mathrm{max}} \\
    &\le& \cc(S_t^*) \log\left(\frac{g_{t,\mathrm{max}}}{g_{t+\ell,\mathrm{max}}}\right) + \cc_{\mathrm{max}} \\
    & \le & \cc(S_t^*) \log\left(\beta_{(i)}\frac{|\overline{U}_{(i-1)}| + (1 + \varepsilon_{\mathrm{rel}})^2|\overline{M}_{(i-1)}|}{\deltabar^2\eta_{(i)}^2}\right) + \cc_{\mathrm{max}} \\
    & \le & C^{*}\left(\frac{\eta_{(i)}}{\beta_{(i)}^{1/2}}, \frac{\eta_{(i)}}{\beta_{(i)}^{1/2}(1+\varepsilon_{\mathrm{rel}})}; \overline{U}_{(i-1)}, \overline{M}_{(i-1)}\right) \times\\
    & & \qquad \log\left(\beta_{(i)}\frac{|\overline{U}_{(i-1)}| + (1 + \varepsilon_{\mathrm{rel}})^2|\overline{M}_{(i-1)}|}{\deltabar^2\eta_{(i)}^2}\right) + \cc_{\mathrm{max}}\\
    & = & C_{(i)}\text{.}
\end{eqnarray*}

Finally, let $i_t$ denote the epoch active at iteration $t$. Since every evaluation costs at least
$\cc_{\min}$, after $t$ evaluations we have
\[
t \le \frac{1}{\cc_{\min}} \sum_{s=1}^t \cc(\lambda_s)
   \le \frac{1}{\cc_{\min}} \sum_{i'=1}^{i_t} C_{(i')}.
\]
Hence \eqref{eq:beta-i-bound} implies
\[
\beta_{(i_t)} \ge 2 \log \frac{|\candset| t^2 \pi^2}{6\delta}.
\]
Applying Lemma~\ref{lemma:srinivas2012information} with $D=\candset$ yields
$\Prob(\mathcal{E}) \ge 1-\delta$.

Now let
\[
i^{*} := \min\{ i \ge 1 : 8(1+\deltabar)\eta_{(i)} \le \epsilon\}.
\]
After cumulative cost $\sum_{i'=1}^{i^{*}} C_{(i')}$, epoch $i^{*}$ has completed. On the event
$\mathcal{E}$, a true minimizer $\lambda^{*} \in \arg\min_{\lambda \in \candset} c(\lambda)$
remains in $M_t$ for all $t$, and therefore
\[
h \in [h_t^\mathrm{opt}, h_t^\mathrm{pes}]
\qquad\text{for all } t.
\]
Using Lemma~\ref{lemma:hard-sets}, after epoch $i^{*}$ we have
$U_t \subseteq \overline{U}_{(i^{*})}$, so every $\lambda \in U_t$ satisfies $|c(\lambda)-h| \le 4(1+\deltabar)\eta_{(i^{*})} \le \epsilon/2$.
Moreover, for all $\lambda \in L_t$, $c(\lambda) \le u_t(\lambda) \le h_t^\mathrm{opt} \le h$ and for all $\lambda \in H_t$, $c(\lambda) \ge l_t(\lambda) > h_t^\mathrm{pes} \ge h$.
Thus $(L_t,H_t,U_t)$ is $\epsilon$-accurate.\qed
\end{proof}

\clearpage

\subsection{Application Notes}

The theorem shows convergence in an idealized setting where $c(\cdot)$ is drawn from the GP prior and observed with i.i.d.\ Gaussian noise.
Most importantly, it makes use of a fixed kernel, whereas GP kernels are typically re-adjusted in applications such as HPO.
The theoretical treatment therefore aids with understanding the algorithm, and how it behaves:
Compared to an algorithm that \emph{optimally} chooses points $\lambda_t$ to evaluate for variance reduction, \TruVarImp{} has only logarithmic slowdown.
In practice, the hyperparameters $\deltabar$ and $\beta_{(i)}$ can be chosen with values that differ from the theory; $\deltabar=0$ and constant $\beta_{(i)}$ work well in practice.
These are also the values we chose in our experiments, see Appendix~\ref{appendix_truvarimp_config}. 

\section{Extended Experiments and Results} \label{appendix_exp_results}

This section contains additional experimental results for Section \ref{sec_experiments}, a formal definition of Rashomon capacity together with our adaptation on regression settings, a formal definition of the FI method permutation feature importance (PFI), a guideline for its interpretation, and a formal definition of variable importance clouds (VICs).

Table \ref{tab_exp_model_CS} displays the number of models per model class for each of the considered tasks. 
Not all model classes are represented in the \CashomonSet{} for every task; for some tasks, only one model class spans the final set.
The star behind a number indicates that the reference model of the corresponding \CashomonSet{} belongs to this model class, underscoring that different model classes should be considered for different tasks. 
For all considered tasks, the reference model belongs to the most common model class in the \CashomonSet{}.
\begin{table}[htb]
\centering
\caption{Number of models of each model class within a \CashomonSet{} found by \TruVarImp per task. ``*'' marks the reference model class.}
\label{tab_exp_model_CS}
\begin{tabular}{llllllll}
  \hline
 & \texttt{glmnet} & \texttt{gosdt} & \texttt{nnet} & \texttt{svm.linear} & \texttt{svm.radial} & \texttt{cart} & \texttt{xgb} \\
  \hline
  \texttt{BC} & 234* & 0 & 130 & 0 & 0 & 0 & 0 \\
  \texttt{BS} & 0 & 0 & 0 & 0 & 0 & 0 & 346* \\
  \texttt{CR} & 0 & 0 & 20* & 0 & 0 & 0 & 0 \\
  \texttt{CS} & 40 & 0 & 9 & 0 & 0 & 2 & 297* \\
  \texttt{CS} (binarized) & 22 & 19 & 66 & 5 & 15 & 33 & 178* \\
  \texttt{FC} & 33 & 0 & 2 & 5 & 1 & 10 & 288* \\
  \texttt{GC} & 61* & 0 & 7 & 106 & 29 & 0 & 154 \\
  \texttt{MK} & 0 & 0 & 10* & 0 & 0 & 0 & 0 \\
  \texttt{ST} & 0 & 0 & 328* & 0 & 17 & 0 & 0 \\
   \hline
\end{tabular}
\end{table}

\clearpage

\subsection{Further Details on Level Set Estimation}\label{appendix_truvarimp_config}
In each GP-based method, we use a Matérn-5/2 kernel with automated relevance detection \cite{rasmussen-book06a}.
We re-estimate lengthscales by maximum likelihood at each step, as is commonly done for HPO.
To stabilize GP HPO, each run starts with 30 random evaluations per model class.

Like Bogunovic et al.~\cite{bogunovic-nips16a}, we set $\beta_t^{1/2}=3$, $\eta_{(1)}=1$, $r=0.1$ and $\deltabar=0$.
We also implement their optimization of evaluating the acquisition function only on the points in $M_{t-1}$ and $U_{t-1}$.
Because we refit the GP in every iteration, $\sigma_t(\lambda)$ is actually non-monotonic with respect to $t$.
We therefore use non-monotonic $M_{t}$ and $U_{t}$, which means that points that have left these sets can be re-added in later iterations.
This is implemented by iterating over all of $\candset$ instead of $U_{t-1}$ or $M_{t-1}$ in Lines \ref{alg:uloop} and \ref{alg:mloop} in Algorithm~\ref{alg:truvarimp}, respectively.
Other algorithms (\texttt{LSE}, $\texttt{LSE}_{\texttt{IMP}}$ and \TruVar) are likewise run in a non-monotonic version.

For the implementation of \texttt{LSE} and $\texttt{LSE}_{\texttt{IMP}}$~\cite{gotovos2013}, we likewise choose $\beta_{t}^{1/2}=3$ as used by the authors.

\texttt{OPTIMIZE} performs Bayesian optimization using the expected improvement acquisition function.

\texttt{LSE}, \texttt{STRADDLE}, and \TruVar are not implicit set-estimation algorithms; i.e., they ordinarily require an explicit cutoff value to classify points.
We run them in a slightly modified form as baseline algorithms, setting the cutoff relative to the best score observed in their optimization run so far.
For this, it is again beneficial that we run these algorithms on non-monotonic candidate point sets. 

\begin{figure}[htb]
  \centering
  \includegraphics[width=\linewidth]{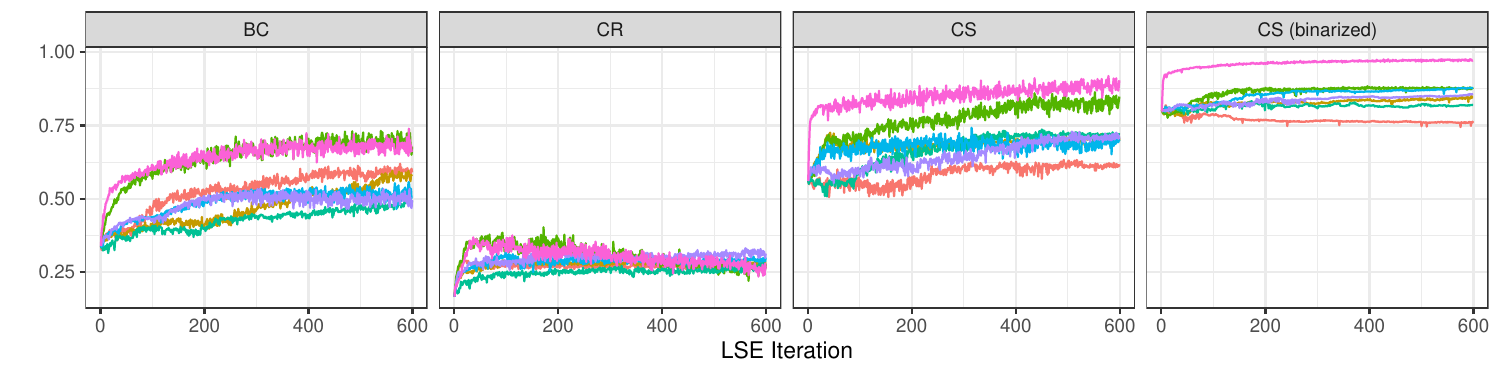}
  \includegraphics[width=\linewidth]{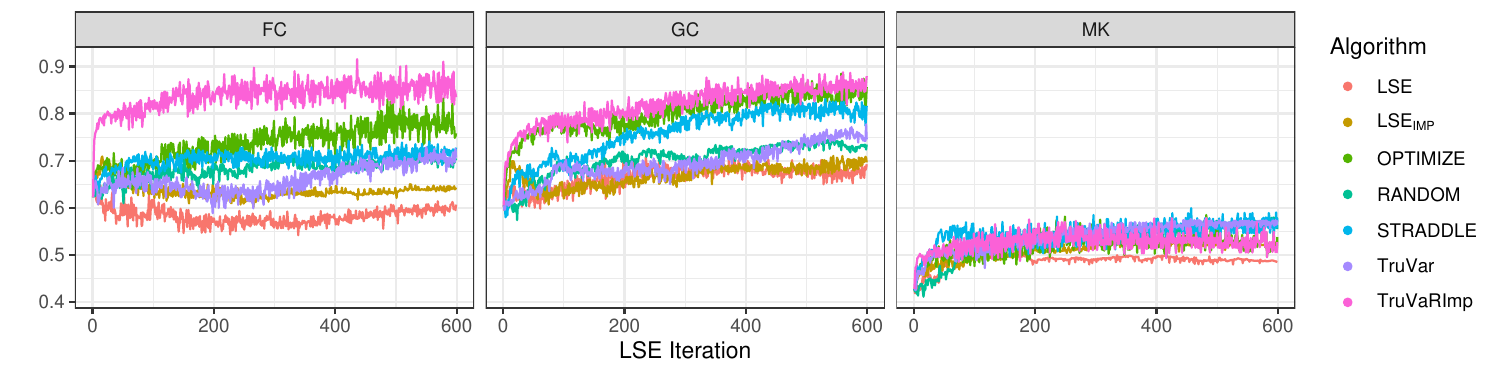}
  \caption{Mean F1 score of the surrogate model predicting \CashomonSet{} membership.}
  \label{fig:cash_results_app}
\end{figure}

\clearpage

\subsection{Rashomon Capacity} \label{appendix_RC}

\emph{Rashomon capacity} (RC) was introduced by Hsu et al. \cite{hsu-neurips22a} as a per-sample metric for predictive multiplicity in classification.
In our notation, for a fixed input $\xv$ and a \RashomonSet{} $\Rset$, the original definition can be written as
\begin{align*}
    m_C(\xv)
    := 2^{C(\xv)}, \text{ with }
    C(\xv)
    := \sup_{P_{\Rset}}
    \inf_{\mathbf q \in \Delta^{g-1}}
    \E_{f \sim P_{\Rset}}
    \left[
        d_{KL}\!\left(f(\xv)\,\|\,\mathbf q\right)
    \right],
\end{align*}
where $\Delta^{g-1}$ denotes the probability simplex over the $g$ classes, and the supremum is taken over all probability measures $P_{\Rset}$ on $\Rset$.
Thus, the original RC measures, for a single observation $\xv$, how much the class-probability predictions of models in $\Rset$ can spread. \cite{hsu-neurips22a}

We adapt RC in two ways: we extend it to regression by replacing KL divergence with the L2 loss, and we formulate a dataset-level variant that simultaneously optimizes a single global distribution $P_{\Rset}$ over models across all observations.
This global formulation models the scenario where a practitioner selects one model from the \RashomonSet{} for deployment.
We model this selection process as a distribution $P_{\Rset}$ over models, which would encode further (unknown) preferences and constraints
by the practitioner.
Since this distribution is unknown, we take a worst-case (or adversarial) perspective and take the supremum over all such distributions to measure the worst-case prediction spread:
\begin{align*}
    C^{\text{classif}} &= \sup_{P_{\Rset}} \E_{(\xv,y)\sim\PXY}  \left[\inf_{\textbf{q}\in \R^g} \E_{f\sim P_{\Rset}}\left[ d_{KL} \left(f(\xv) \,\|\, \textbf{q}\right) \right] \right], \text{ and} \\
    C^{\text{regr}} &= \sup_{P_{\Rset}} \E_{(\xv,y)\sim\PXY}  \left[\inf_{q\in \R} \E_{f\sim P_{\Rset}}\left[ \left(f(\xv) - q\right)^2 \right] \right].
\end{align*}
In practice we usually consider finite approximations of \RashomonSets{}, 
hence our distribution $P_{\Rset}$ becomes categorical and is represented by a vector of weights $w_m$ 
(one weight for each model $f_m$ in the \RashomonSet{}), 
with $0 \leq w_m \leq 1$ and $\sum_{m = 1}^M w_m = 1$.
A key result from information geometry states that for a mixture of discrete distributions, the unique minimizer of this right-sided divergence is the weighted arithmetic mean of the components \cite{nielsen2009}.
By substituting \textbf{q} with the weighted arithmetic mean in the objective, our optimization problem simplifies to:
\begin{align}
    C^{\text{classif}} &= \max_{\mathbf{w}} \frac{1}{n} \sum_{i=1}^n \left[ H\left( \sum_{m=1}^{M} w_m f_m(\xv^{(i)}) \right) - \sum_{m=1}^{M} w_m H\left(f_m(\xv^{(i)})\right) \right], \\
    C^{\text{regr}} &= \max_{\mathbf{w}} \frac{1}{n} \sum_{i=1}^n \left[  \sum_{m=1}^{M} w_m \left(f_m(\xv^{(i)})\right)^2 - \left(\sum_{m=1}^{M} w_m f_m(\xv^{(i)})\right)^2 \right], \\
    & \text{subject to } w_m \in [0, 1], \sum_{m=1}^M w_m = 1, \notag
\end{align}
where $H$ is the Shannon entropy.
Both objectives are concave in $\mathbf{w}$: for regression, the objective is linear minus convex-quadratic (a concave quadratic program); for classification, it is the concave entropy of a mixture minus a linear term (a concave program).
Concavity over the convex constraint set (probability simplex) guarantees that any local maximum is also a global one.
We use the \texttt{CVXR} package \cite{Fu-cvxr} in \texttt{R} to solve the optimization problem via disciplined convex programming.
Algorithm~\ref{alg:rashomon_capacity_agnostic} shows how to construct the symbolic graph for \texttt{CVXR}.

\begin{algorithm}[htb]
\algrenewcommand\algorithmicrequire{\textbf{Input:}}
\algrenewcommand\algorithmicensure {\textbf{Output:}}
	\caption{Rashomon Capacity via Symbolic Graph Construction}
	\label{alg:rashomon_capacity_agnostic}
	\begin{algorithmic}[1]
		\Require Models $\{f_m\}_{m=1}^M$ (regression or $g$-class classification), observations $\{\xi\}_{i=1}^n$
        \Ensure Rashomon capacity metric
		\State Let $\mathbf{w} \triangleq (w_1, \dots, w_M)$ be the symbolic vector of model weights

		\If {models $f_m$ are regression models}
			\State Predict $\yh_{im} \leftarrow f_m(\xi)$ for all $i, m$
			\For {$i = 1, \dots, n$}
                \State Define $\overline{y}_i(\mathbf{w}) \triangleq \sum_{m=1}^{M} w_m \yh_{im}$ \Comment{Expected prediction}
				\State Define $V_i(\mathbf{w}) \triangleq \sum_{m=1}^{M} w_m (\yh_{im})^2 - (\overline{y}_i(\mathbf{w}))^2$ \Comment{Variance expression}
			\EndFor
			\State Define $\text{Obj}(\mathbf{w}) \triangleq \sum_{i=1}^{n} V_i(\mathbf{w})$
		\Else \Comment{Classification Objective}
			\State Predict $\hat{p}_{imk} \leftarrow [f_m(\xi)]_{k}$\Comment{Predicted $\Prob(\text{class } k \mid \text{model } m, \text{instance } i)$}
			\For {each $i \in \{1, \dots, n\}$ and $m \in \{1, \dots, M\}$}
				\State $H_{im} \leftarrow -\sum_{k=1}^{g} \hat{p}_{imk} \log(\hat{p}_{imk})$ \Comment{Precompute constant base entropies}
			\EndFor
			\State Define $\text{Obj}(\mathbf{w}) \triangleq 0$
			\For {$i = 1, \dots, n$}
				\For {$k = 1, \dots, g$}
					\State Define $\overline{p}_{ik}(\mathbf{w}) \triangleq \sum_{m=1}^{M} w_m \hat{p}_{imk}$\Comment{Weighted probability predictions}
				\EndFor
				\State Define $H_i^Q(\mathbf{w}) \triangleq -\sum_{k=1}^{g} \overline{p}_{ik}(\mathbf{w}) \log(\overline{p}_{ik}(\mathbf{w}))$
				\State Define $\overline{H}_i(\mathbf{w}) \triangleq \sum_{m=1}^{M} w_m H_{im}$ \Comment{Expected entropy}
				\State Define $JSD_i(\mathbf{w}) \triangleq H_i^Q(\mathbf{w}) - \overline{H}_i(\mathbf{w})$\Comment{Jensen-Shannon Divergence}
				\State Update $\text{Obj}(\mathbf{w}) \leftarrow \text{Obj}(\mathbf{w}) + JSD_i(\mathbf{w})$
			\EndFor
		\EndIf

		\State \Comment{Pass the constructed symbolic graph to a convex solver}
		\State $\mathbf{w}^* \leftarrow \argmax_{\mathbf{w} \in \Delta^{M-1}} \left\{ \text{Obj}(\mathbf{w}) \right\} \quad \text{s.t.} \quad w_m \ge 0, \sum_{m=1}^{M} w_m = 1$
		\State \Return Result $\leftarrow \text{Obj}(\mathbf{w}^*) / n$
	\end{algorithmic}
\end{algorithm}

\clearpage

\subsection{Variable Importance Clouds} \label{appendix_vic_results}

A \emph{variable importance cloud} (VIC) for a model set $\Rset$ is defined as
\begin{align*}
    \text{VIC}(\Rset) = \{ \text{MR}(f) \mid f \in \Rset \},
\end{align*}
where $\text{MR}(f) \in \mathbb{R}^p$ is a vector of \emph{model reliance} values, one per feature.
Each element $\text{MR}_j(f)$ quantifies the degree to which model $f$ relies on feature $X_j$ (i.e., a feature importance).
Visualized as a scatter plot with one point per feature per model, the VIC reveals the full distribution of feature importances across all models in the set.
In our paper, we use PFI as the model reliance measure, defined below.

\paragraph{Permutation Feature Importance}
PFI \cite{breiman_statistical_2001,fisher-jmlr19a} is a feature importance method that assigns to each feature a single score, which can then be compared across features to assess their relative importance.
For a given model $f$, PFI for a feature of interest $X_j$ is defined as
\begin{align*}
    \text{PFI}_j(f) =\, &\E_{(X,Y)\sim\PXY,\Tilde{X}_j\sim P_{X_j}}\left[L\left(Y, f(\Tilde{X}_{j}, X_{-j})\right)\right] -\E_{(X,Y)\sim\PXY}\left[L\left(Y, f(X)\right)\right],
\end{align*}
where $X_{-j}$ denotes the feature vector $X$ without the feature of interest, $X_j,$ and the permuted feature $\Tilde{X}_{j}$ is distributed according to the marginal distribution $P_{X_j}$.
This quantifies how the model relies on $X_j$, and, if $X_j$ is conditionally independent of all other features given $Y$, how $X_j$ is associated with the target $Y$ \cite{Ewald-xAI24a}.
In practice, we generally do not have independent features at hand; still, PFI is useful for analyzing how a model relies on the feature.

\paragraph{Further results.}{

In Section \ref{sec:fi_for_RE}, we concentrate on comparing our \TruVarImp \CashomonSet{} with \RashomonSets{} found by \TreeFARMS.
Since our visualizations of feature importance aim only to show application examples of Rashomon and \CashomonSets{}, and due to space constraints, we focus on the most interesting tasks in the main paper.
Figures \ref{fig_vics_all1a} and \ref{fig_vics_all2} show the VICs for the remaining four binary datasets on which we base our comparison.

Notably, task \texttt{MK} differs from the others in the sense that 
\begin{inparaenum}[\bf (1)]
    \item the FI values in the \CashomonSet{} agree across models, and 
    \item the \TreeFARMS \RashomonSet{} contains models assigning negative PFI values to some features. 
\end{inparaenum}
For this dataset, the \emph{no-skill} Brier baseline is $0.2353$, obtained from the class prevalence baseline $\P(Y = 1) \times (1-\P(Y = 1))$. 
Hence, the \TreeFARMS \RashomonSet{} contains models that perform worse than an uninformative constant predictor (Brier score $> 0.2353$, see Figure~\ref{fig_vics_all1b}), explaining why some PFI values are below $0$.
At the same time, \texttt{MK} seems to be a task that can be perfectly solved by a model from model class \texttt{nnet}, which is part of the CASH space but not considered in \TreeFARMS.
As such, only one final model (maybe found several times through different HPCs) is in our \CashomonSet{} with Brier score $=0$ (and, hence, $\varepsilon_{\mathrm{rel}} = 0$).

\begin{figure}[htb]
    \centering
    
    \begin{subfigure}[b]{\textwidth}
        \centering
        \includegraphics[width=\textwidth, trim={0 1cm 0 1.8cm}, clip]{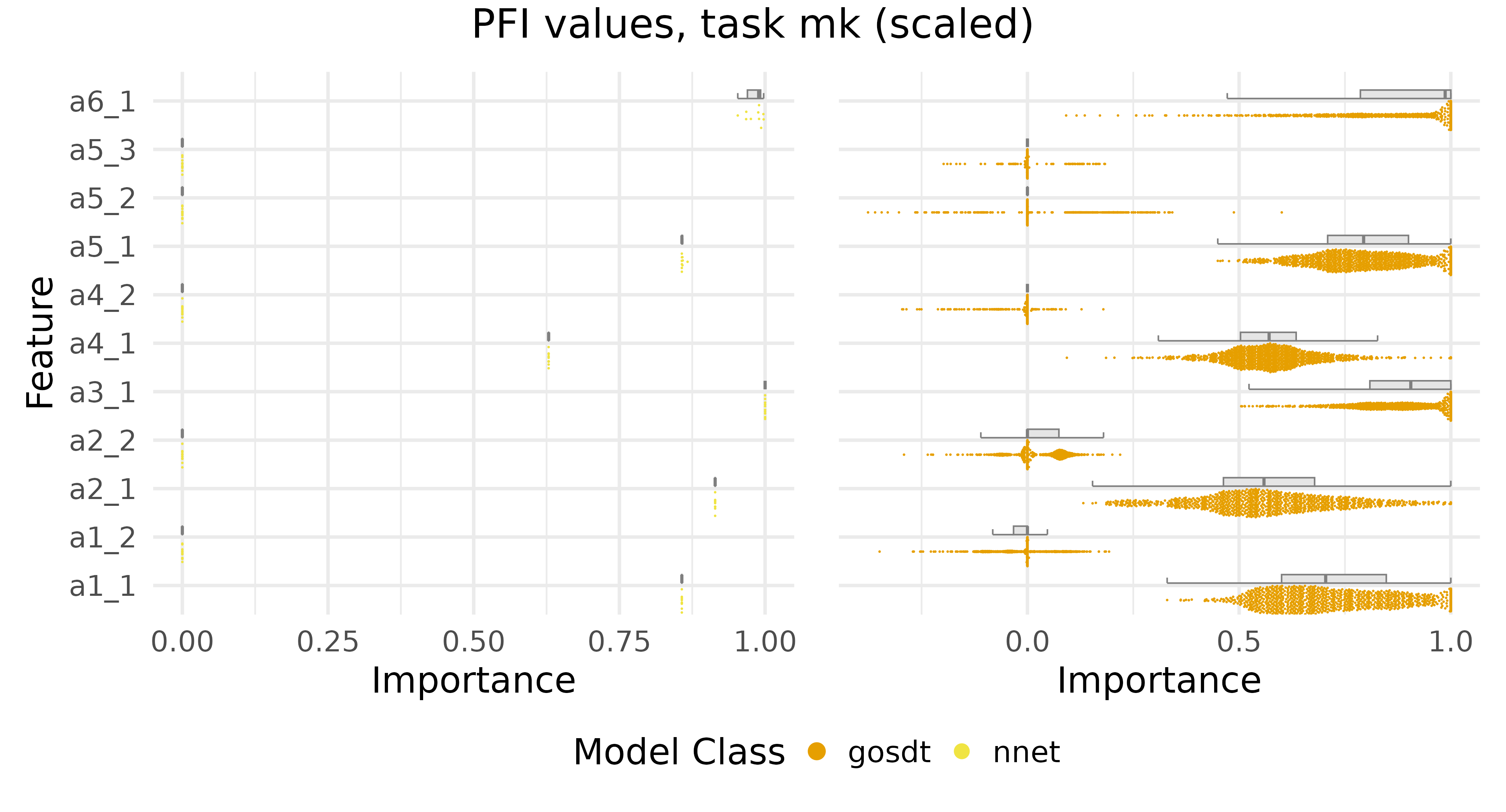}
        \caption{VICs}
    \label{fig_vics_all1a}
    \end{subfigure}
    
    \vspace{0.8em}
    
    \begin{subfigure}[b]{0.7\textwidth}
        \centering
        \includegraphics[width=\textwidth, trim={0 0 0 1.7cm}, clip]{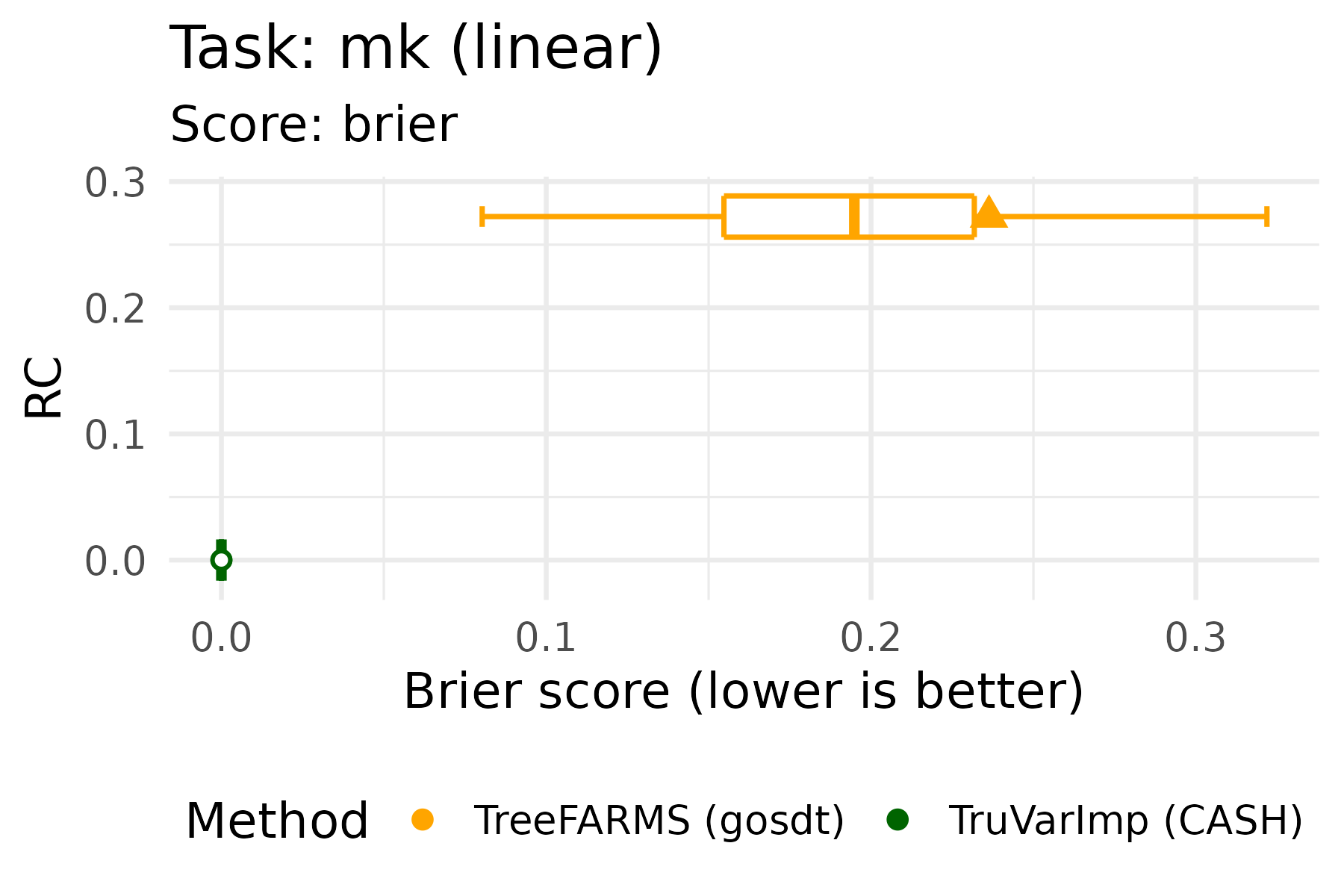}
        \caption{Rashomon capacity vs. predictive performance}
    \label{fig_vics_all1b}
    \end{subfigure}
    
    \caption{(a) VICs for \CashomonSet{} (left) and \TreeFARMS \RashomonSet{} (right) for task \texttt{MK}. For each feature, a point cloud colored by model class and a boxplot are displayed. If several models yield the same importance value, the points in the cloud scatter vertically. PFI values are scaled: the maximal PFI value for a model equals 1. \\
    (b) RC (y-axis) versus Brier score (x-axis); horizontal boxplots summarize test score distributions; triangles indicate the reference model.}
    \label{fig_vics_all1}
\end{figure}

\begin{figure}[htb]
    \centering
    
    \begin{subfigure}[b]{0.8\textwidth}
        \centering
        \includegraphics[width=\textwidth, trim={0 2.6cm 0 1.7cm}, clip]{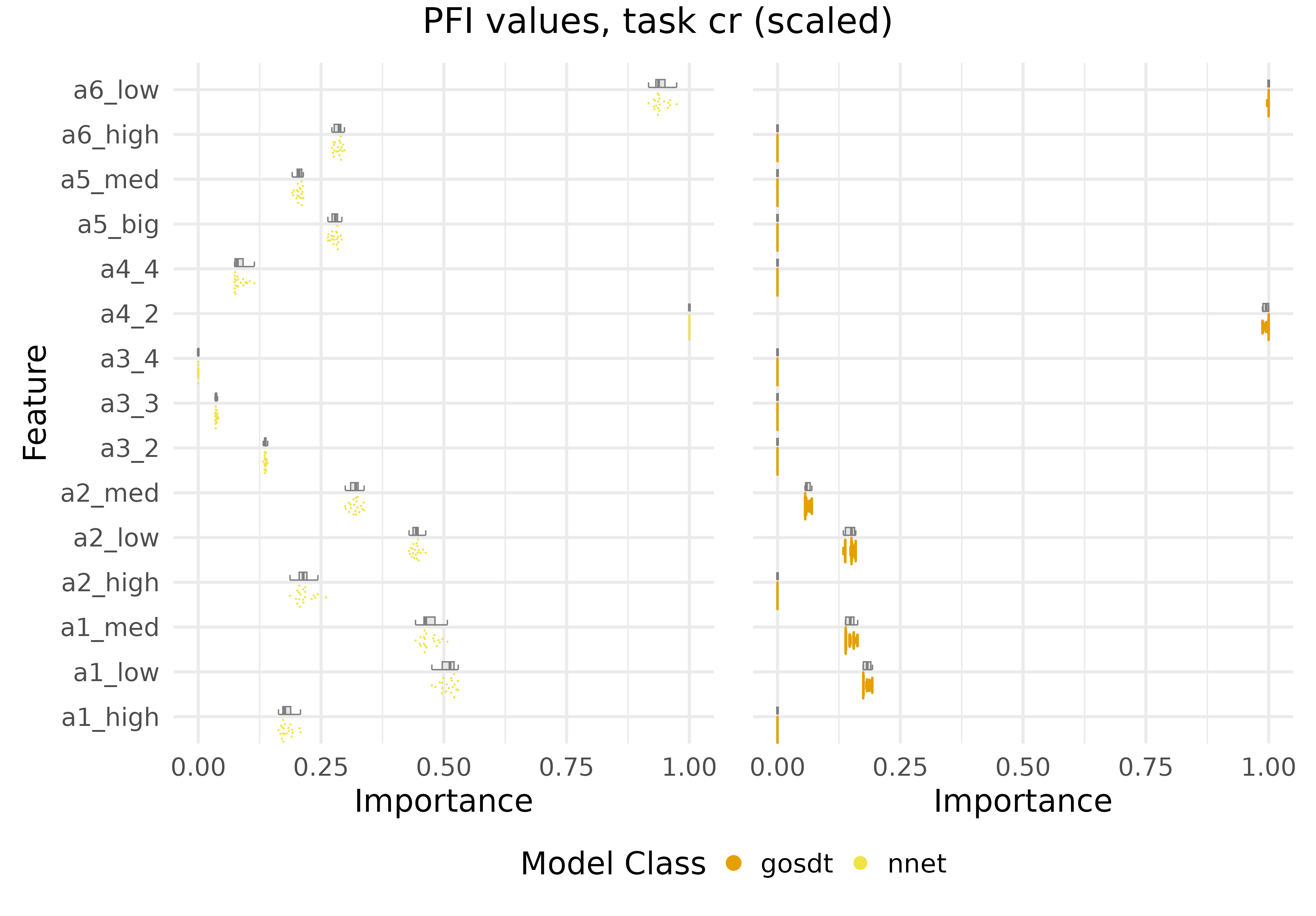}
        \caption{Task \texttt{CR}}
    \end{subfigure}
    
    \vspace{0.8em}
    
    \begin{subfigure}[b]{0.8\textwidth}
        \centering
        \includegraphics[width=\textwidth, trim={0 2.6cm 0 1.7cm}, clip]{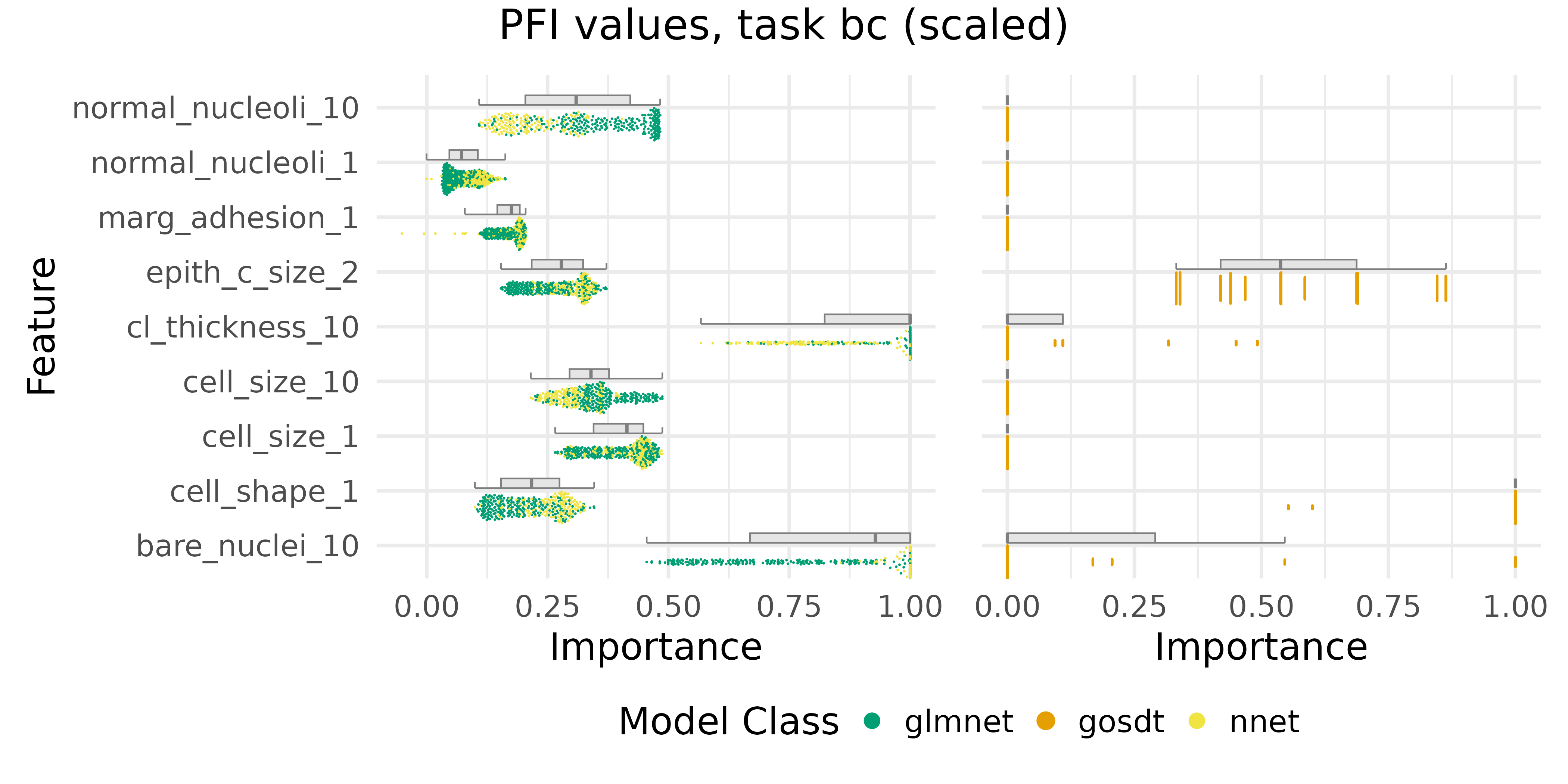}
        \caption{Task \texttt{BC}}
    \end{subfigure}
    
    \vspace{0.8em}
    
    \begin{subfigure}[b]{0.8\textwidth}
        \centering
        \includegraphics[width=\textwidth, trim={0 0 0 1.7cm}, clip]{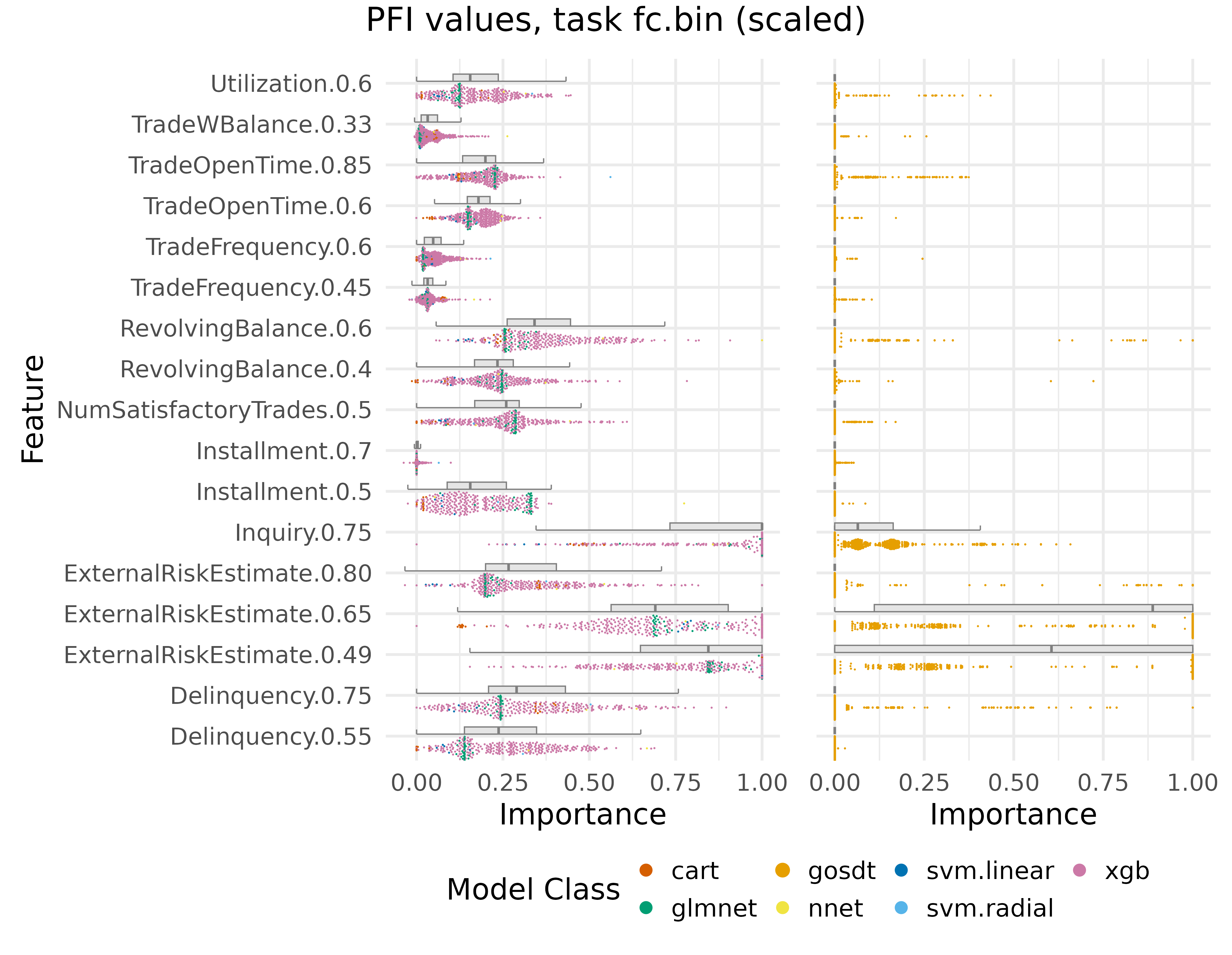}
        \caption{Task \texttt{FC}}
    \end{subfigure}   
    
    \caption{VICs for \CashomonSets{} (left) and \TreeFARMS \RashomonSets{} (right) for tasks \texttt{CR}, \texttt{BC}, and \texttt{FC}. For each feature, a point cloud colored by model class and a boxplot are displayed. If several models yield the same importance value, the points in the cloud scatter vertically. PFI values are scaled: the maximal PFI value for a model equals 1.}
    \label{fig_vics_all2}
\end{figure}

}

\clearpage

\end{bibunit}

\begin{thebibliography}{10}
\providecommand{\url}[1]{\texttt{#1}}
\providecommand{\urlprefix}{URL }
\providecommand{\doi}[1]{https://doi.org/#1}

\bibitem{binder-automl20a}
Binder, M., Pfisterer, F., Bischl, B.: Collecting empirical data about
  hyperparameters for data driven {AutoML}. In: {ICML} {W}orkshop on {A}uto{ML}
  (2020)

\bibitem{bischl-dmkd23a}
Bischl, B., Binder, M., Lang, M., Pielok, T., Richter, J., Coors, S., Thomas,
  J., Ullmann, T., Becker, M., Boulesteix, A., Deng, D., Lindauer, M.:
  Hyperparameter optimization: Foundations, algorithms, best practices, and
  open challenges. Wiley IRDMKD p. e1484 (2023)

\bibitem{bogunovic-nips16a}
Bogunovic, I., Scarlett, J., Krause, A., Cevher, V.: Truncated variance
  reduction: A unified approach to {Bayesian} optimization and level-set
  estimation. NeurIPS  \textbf{29} (2016)

\bibitem{breiman-mlj01a}
Breiman, L.: Random forests. MLJ  \textbf{45},  5--32 (2001)

\bibitem{breiman_statistical_2001}
Breiman, L.: Statistical modeling: The two cultures (with comments and a
  rejoinder by the author). Stat. Sci.  \textbf{16}(3) (2001)

\bibitem{bryan2005active}
Bryan, B., Nichol, R.C., Genovese, C.R., Schneider, J., Miller, C.J.,
  Wasserman, L.: Active learning for identifying function threshold boundaries.
  NeurIPS  \textbf{18} (2005)

\bibitem{cavus-isf26a}
Cavus, M., van Rijn, J.N., Biecek, P.: Quantifying model uncertainty with
  {AutoML} and {Rashomon} partial dependence profiles: Enabling trustworthy and
  human-centered {XAI}. Inf. Syst. Front.  (2026)

\bibitem{ciaperoni-acm24}
Ciaperoni, M., Xiao, H., Gionis, A.: Efficient exploration of the {Rashomon}
  set of rule-set models. KDD pp. 478--489 (2024)

\bibitem{dong-nmi20a}
Dong, J., Rudin, C.: Exploring the cloud of variable importance for the set of
  all good models. Nat. Mach. Intell.  \textbf{2}(12),  810--824 (2020)

\bibitem{fernandez-jmlr14a}
Fern{{\'a}}ndez-Delgado, M., Cernadas, E., Barro, S., Amorim, D.: Do we need
  hundreds of classifiers to solve real world classification problems? JMLR
  \textbf{15}(90),  3133--3181 (2014)

\bibitem{feurer-nips15a}
Feurer, M., Klein, A., Eggensperger, K., Springenberg, J., Blum, M., Hutter,
  F.: Efficient and robust automated machine learning. NeurIPS pp. 2962--2970
  (2015)

\bibitem{fisher-jmlr19a}
Fisher, A., Rudin, C., Dominici, F.: All models are wrong, but many are useful:
  Learning a variable's importance by studying an entire class of prediction
  models simultaneously. {JMLR}  \textbf{20}(177),  1--81 (2019)

\bibitem{garnett-book23a}
Garnett, R.: Bayesian optimization. Cambridge University Press (2023)

\bibitem{gotovos2013}
Gotovos, A., Casati, N., Hitz, G., Krause, A.: Active learning for level set
  estimation. IJCAI pp. 1344--1350 (2013)

\bibitem{hsu-nips2024}
Hsu, H., Brugere, I., Sharma, S., Lecue, F., Chen, R.: {RashomonGB}: Analyzing
  the {Rashomon} effect and mitigating predictive multiplicity in gradient
  boosting. NeurIPS  \textbf{37},  121265--121303 (2024)

\bibitem{hsu-neurips22a}
Hsu, H., Calmon, F.: {Rashomon} capacity: A metric for predictive multiplicity
  in classification. NeurIPS pp. 28988--29000 (2022)

\bibitem{hsu_arxiv24}
Hsu, H., Li, G., Hu, S., et~al.: Dropout-based {Rashomon} set exploration for
  efficient predictive multiplicity estimation. arXiv:2402.00728  (2024)

\bibitem{hutter-book19a}
Hutter, F., Kotthoff, L., Vanschoren, J.: Automated machine learning: Methods,
  systems, challenges. Springer (2019)

\bibitem{Kobylinska-ieee24a}
Kobyli\'{n}ska, K., Krzyzi\'{n}ski, M., Machowicz, R., Adamek, M., Biecek, P.:
  Exploration of the {Rashomon} set assists trustworthy explanations for
  medical data. IEEE J. Biomed. Health Inform.  \textbf{28}(11),  6454--6465
  (2024)

\bibitem{laberge-jmlr23a}
Laberge, G., Pequignot, Y., Mathieu, A., Khomh, F., Marchand, M.: Partial order
  in chaos: consensus on feature attributions in the {Rashomon} set. {JMLR}
  \textbf{24}(364),  1--50 (2023)

\bibitem{ledell-automl20a}
LeDell, E., Poirier, S.: {H2O} {AutoML}: Scalable automatic machine learning.
  In: {ICML} {W}orkshop on {A}uto{ML} (2020)

\bibitem{lin-icml20a}
Lin, J., Zhong, C., Hu, D., Rudin, C., Seltzer, M.: Generalized and scalable
  optimal sparse decision trees. ICML pp. 6150--6160 (2020)

\bibitem{mata_arXiv22}
Mata, K., Kanamori, K., Arimura, H.: Computing the collection of good models
  for rule lists. arXiv:2204.11285  (2022)

\bibitem{mcelfresh-neurips23a}
McElfresh, D., Khandagale, S., Valverde, J., {Prasad C}, V., Ramakrishnan, G.,
  Goldblum, M., White, C.: When do neural nets outperform boosted trees on
  tabular data? NeurIPS pp. 76336--76369 (2023)

\bibitem{mueller-ecml23a}
Müller, S., Toborek, V., Beckh, K., Jakobs, M., Bauckhage, C., Welke, P.: An
  empirical evaluation of the {Rashomon} effect in explainable machine
  learning. ECML PKDD  \textbf{14171},  462--478 (2023)

\bibitem{rasmussen-book06a}
Rasmussen, C., Williams, C.: Gaussian processes for machine learning. The MIT
  Press (2006)

\bibitem{rudin-icml24a}
Rudin, C., Zhong, C., Semenova, L., Seltzer, M., Parr, R., Liu, J., Katta, S.,
  Donnelly, J., Chen, H., Boner, Z.: Position: Amazing things come from having
  many good models. ICML pp. 42783--42795 (2024)

\bibitem{semenova-facct22a}
Semenova, L., Rudin, C., Parr, R.: On the existence of simpler machine learning
  models. ACM FAccT pp. 1827--1858 (2022)

\bibitem{shwartz-ziv-if22a}
Shwartz-Ziv, R., Armon, A.: Tabular data: Deep learning is not all you need.
  Information Fusion  \textbf{81},  84--90 (2022)

\bibitem{thornton-kdd13a}
Thornton, C., Hutter, F., Hoos, H., Leyton-Brown, K.: {A}uto-{WEKA}: combined
  selection and hyperparameter optimization of classification algorithms. KDD
  pp. 847--855 (2013)

\bibitem{xin-neurips22a}
Xin, R., Zhong, C., Chen, Z., Takagi, T., Seltzer, M., Rudin, C.: Exploring the
  whole {Rashomon} set of sparse decision trees. NeurIPS pp. 14071--14084
  (2022)

\bibitem{zhong-neurips23a}
Zhong, C., Chen, Z., Liu, J., Seltzer, M., Rudin, C.: Exploring and interacting
  with the set of good sparse generalized additive models. NeurIPS pp.
  56673--56699 (2023)

\end{thebibliography}

\begin{thebibliography}{10}
\providecommand{\url}[1]{\texttt{#1}}
\providecommand{\urlprefix}{URL }
\providecommand{\doi}[1]{https://doi.org/#1}

\bibitem{angwin2016machine}
Angwin, J., Larson, J., Mattu, S., Kirchner, L.: Machine bias --- there's
  software used across the country to predict future criminals. and it's biased
  against blacks. ProPublica, Online Edition (2016)

\bibitem{binder-automl20a}
Binder, M., Pfisterer, F., Bischl, B.: Collecting empirical data about
  hyperparameters for data driven {AutoML}. In: {ICML} {W}orkshop on {A}uto{ML}
  (2020)

\bibitem{mlr3pipelines}
Binder, M., Pfisterer, F., Lang, M., Schneider, L., Kotthoff, L., Bischl, B.:
  {mlr3pipelines} - flexible machine learning pipelines in {R}. {JMLR}
  \textbf{22}(184), ~1--7 (2021)

\bibitem{bischl-pattern25a}
Bischl, B., Casalicchio, G., Das, T., Feurer, M., Fischer, S., Gijsbers, P.,
  Mukherjee, S., Müller, A.C., Németh, L., Oala, L., Purucker, L., Ravi, S.,
  {van Rijn}, J.N., Singh, P., Vanschoren, J., {van der Velde}, J., Wever, M.:
  {OpenML}: Insights from 10 years and more than a thousand papers. Patterns
  \textbf{6}(7),  101317 (2025)

\bibitem{bogunovic-nips16a}
Bogunovic, I., Scarlett, J., Krause, A., Cevher, V.: Truncated variance
  reduction: A unified approach to {Bayesian} optimization and level-set
  estimation. NeurIPS  \textbf{29} (2016)

\bibitem{cardatasett}
Bohanec, M., Rajkovi\v{c}, V.: Knowledge acquisition and explanation for
  multi-attribute decision making. In: 8th Intl Workshop on Expert Systems and
  their Applications (1988)

\bibitem{breiman_statistical_2001}
Breiman, L.: Statistical modeling: The two cultures (with comments and a
  rejoinder by the author). Stat. Sci.  \textbf{16}(3) (2001)

\bibitem{brier1950verification}
Brier, G.W.: Verification of forecasts expressed in terms of probability. Mon.
  Weather Rev.  \textbf{78}(1), ~1--3 (1950)

\bibitem{chen-kdd16a}
Chen, T., Guestrin, C.: {XGBoost}: {A} scalable tree boosting system. In: Proc.
  of {KDD}'16. pp. 785--794 (2016)

\bibitem{Ewald-xAI24a}
Ewald, F., Bothmann, L., Wright, M., Bischl, B., Casalicchio, G., K{\"o}nig,
  G.: A guide to feature importance methods for scientific inference. In: xAI.
  pp. 440--464. Springer Nature Switzerland, Cham (2024)

\bibitem{fanaee-pai14}
Fanaee-T, H., Gama, J.: Event labeling combining ensemble detectors and
  background knowledge. Prog. Artif. Intell.  \textbf{2},  113--127 (2014)

\bibitem{fisher-jmlr19a}
Fisher, A., Rudin, C., Dominici, F.: All models are wrong, but many are useful:
  Learning a variable's importance by studying an entire class of prediction
  models simultaneously. {JMLR}  \textbf{20}(177),  1--81 (2019)

\bibitem{Fu-cvxr}
Fu, A., Narasimhan, B., Boyd, S.: {CVXR}: An {R} package for disciplined convex
  optimization. J. Stat. Softw.  \textbf{94}(14),  1--34 (2020)

\bibitem{gotovos2013}
Gotovos, A., Casati, N., Hitz, G., Krause, A.: Active learning for level set
  estimation. IJCAI pp. 1344--1350 (2013)

\bibitem{data-gc}
Hofmann, H.: Statlog (german credit data) (1994)

\bibitem{hsu-neurips22a}
Hsu, H., Calmon, F.: {Rashomon} capacity: A metric for predictive multiplicity
  in classification. NeurIPS pp. 28988--29000 (2022)

\bibitem{krause2014submodular}
Krause, A., Golovin, D.: Submodular function maximization. Tractability
  \textbf{3}(71-104), ~3 (2014)

\bibitem{krause2005note}
Krause, A., Guestrin, C.: A note on the budgeted maximization of submodular
  functions. Citeseer (2005)

\bibitem{lin-icml20a}
Lin, J., Zhong, C., Hu, D., Rudin, C., Seltzer, M.: Generalized and scalable
  optimal sparse decision trees. ICML pp. 6150--6160 (2020)

\bibitem{Meyer1999}
Meyer, D., Dimitriadou, E., Hornik, K., Weingessel, A., Leisch, F.: {e1071}:
  Misc functions of the department of statistics, probability theory group
  (formerly: {E1071}), tu wien (1999)

\bibitem{nielsen2009}
Nielsen, F., Nock, R.: Sided and symmetrized {Bregman} centroids. IEEE Trans.
  Inf. Theory.  \textbf{55}(6),  2882--2904 (2009)

\bibitem{rasmussen-book06a}
Rasmussen, C., Williams, C.: Gaussian processes for machine learning. The MIT
  Press (2006)

\bibitem{Ripley2009}
Ripley, B.: {nnet}: Feed-forward neural networks and multinomial log-linear
  models (2009)

\bibitem{srinivas2012information}
Srinivas, N., Krause, A., Kakade, S.M., Seeger, M.W.: Information-theoretic
  regret bounds for {Gaussian} process optimization in the bandit setting. IEEE
  Trans. Inf. Theory.  \textbf{58}(5),  3250--3265 (2012)

\bibitem{Therneau1999}
Therneau, T., Atkinson, B.: {rpart}: Recursive partitioning and regression
  trees (1999)

\bibitem{monk2dataset}
Thrun, S., Bala, J., Bloedorn, E., Bratko, I., Cestnik, B., Cheng, J., Jong,
  K.D., Dzeroski, S., Fahlman, S., Fisher, D., Hamann, R., Kaufman, K., Keller,
  S., Kononenko, I., Kreuziger, J., Michalski, R., Mitchell, T., Pachowicz, P.,
  Reich, Y., Vafaie, H., de~Welde, W.V., Wenzel, W., Wnek, J., , Zhang, J.: The
  {MONK's} problems - a performance comparison of different learning
  algorithms. Tech. Rep. CS-CMU-91-197, Carnegie Mellon University (1991)

\bibitem{breastcancerdataset}
Wolberg, W., Mangasarian, O.: Multisurface method of pattern separation for
  medical diagnosis applied to breast cytology. Proc Natl Acad Sci U S A
  \textbf{87}(23),  9193--91999 (1990)

\bibitem{xin-neurips22a}
Xin, R., Zhong, C., Chen, Z., Takagi, T., Seltzer, M., Rudin, C.: Exploring the
  whole {Rashomon} set of sparse decision trees. NeurIPS pp. 14071--14084
  (2022)

\bibitem{zou-jrsssb05a}
Zou, H., Hastie, T.: Regularization and variable selection via the elastic net.
  J. R. Stat. Soc. Ser. B Stat. Methodol.  \textbf{67}(2),  301--320 (2005)

\end{thebibliography}
\end{document}